\documentclass[lettersize,journal]{IEEEtran}
\usepackage{amsmath,amsfonts}
\usepackage{algorithmic}
\usepackage{algorithm}
\usepackage{array}
\usepackage[caption=false,font=normalsize,labelfont=sf,textfont=sf]{subfig}
\usepackage{textcomp}
\usepackage{stfloats}
\usepackage{url}
\usepackage{verbatim}
\usepackage{graphicx}
\usepackage{cite}
\usepackage{xcolor}         % colors
\usepackage{enumitem}
\setlist{labelindent=\parindent, leftmargin=*}

\usepackage{wrapfig}
\usepackage{booktabs}
\usepackage{multirow}
\usepackage{graphicx}
\usepackage{wrapfig}
\usepackage{amsmath}
\usepackage{amsthm}
\newtheorem{theorem}{Theorem}
\newtheorem{lemma}[theorem]{Lemma}
\newtheorem{proposition}{Proposition}
\begin{document}

\title{Rethinking Spectral Graph Neural Networks with Spatially Adaptive Filtering}

\author{Jingwei~Guo,
        Kaizhu~Huang*,
        Xinping~Yi,
        Zixian~Su,
        and~Rui~Zhang
\IEEEcompsocitemizethanks{
\IEEEcompsocthanksitem J. Guo is with University of Liverpool, Liverpool, UK and is also with Xi’an Jiaotong-Liverpool University, Suzhou, China (E-mail:~jingweiguo19@outlook.com); K. Huang is with Duke Kunshan University, Suzhou, China(E-mail:~kaizhu.huang@dukekunshan.edu.cn); X. Yi is with Southeast University, Nanjing, China(E-mail:~xyi@seu.edu.cn); Z. Su is with University of Liverpool, Liverpool, UK and is also with Xi’an Jiaotong-Liverpool University, Suzhou, China (E-mail:~zixian.su@liverpool.ac.uk); R. Zhang is with Xi’an Jiaotong-Liverpool University, Suzhou, China(
E-mail:~rui.zhang02@xjtlu.edu.cn); *Corresponding author: Kaizhu Huang.}
}



\maketitle

\begin{abstract}
Whilst spectral Graph Neural Networks (GNNs) are theoretically well-founded in the spectral domain, their practical reliance on polynomial approximation implies a profound linkage to the spatial domain. As previous studies rarely examine spectral GNNs from the spatial perspective, their spatial-domain interpretability remains elusive, \textit{e.g.}, what information is essentially encoded by spectral GNNs in the spatial domain? In this paper, to answer this question, we investigate the theoretical connection between spectral filtering and spatial aggregation, unveiling an intrinsic interaction that spectral filtering implicitly leads the original graph to an adapted new graph, explicitly computed for spatial aggregation. Both theoretical and empirical investigations reveal that the adapted new graph not only exhibits non-locality but also accommodates signed edge weights to reflect label consistency among nodes. These findings highlight the interpretable role of spectral GNNs in the spatial domain and inspire us to rethink graph spectral filters beyond the fixed-order polynomials, which limit the effective propagation range and hinder their ability to capture long-range dependencies. Built upon the theoretical findings, we revisit the state-of-the-art spectral GNNs and propose a novel Spatially Adaptive Filtering (SAF) framework, which leverages the adapted new graph by spectral filtering for an auxiliary non-local aggregation. Notably, our SAF comprehensively models both node similarity and dissimilarity from a global perspective, therefore alleviating persistent deficiencies of GNNs related to long-range dependencies and graph heterophily. Extensive experiments over 13 node classification benchmarks demonstrate the superiority of our proposed framework to the state-of-the-art methods.
\end{abstract}


\section{Introduction}
Graph Neural Networks (GNNs) have shown remarkable abilities to uncover the intricate dependencies within graph-structured data, and achieved tremendous success in graph machine learning~\cite{shanthamallu2019gramme,Wu2020ACS,ji2021survey,chami2022machine,gao2023survey,gravina2024deep}. Spectral GNNs are a class of GNNs rooted in spectral graph theory~\cite{chung1997spectral,shuman2013emerging}, implementing graph convolutions via spectral filters~\cite{defferrard2016convolutional,kipf2017semi}. Whilst various spectral filtering strategies~\cite{levie2018cayleynets,chien2021adaptive,he2021bernnet,ju2022adaptive,he2022chebnetii,Wang2022HowPA,guo2023graph,tao2023longnn,guo2023OptBasisGNN} have been proposed for spectral GNNs, their practical implementations invariably resort to approximating graph filters with fixed-order polynomials for computational efficiency~\cite{Wang2022HowPA,he2022chebnetii}. This truncated approach essentially relies on the direct extraction of spatial features from the local regions of nodes. As such, the spatial domain of a graph, albeit loosely connected to spectral GNNs in theory, still plays a crucial role in effectively learning node representations.

However, there is a notable lack of research examining spectral GNNs from the spatial perspective. Though recent studies analyze both spectral and spatial GNNs to elucidate their similarities in model formulations~\cite{balcilar2020bridging,chen2021bridging,guo2023clenshaw}, outcomes~\cite{ma2021unified,zhu2021interpreting}, and expressiveness~\cite{chen2020simple,balcilar2021analyzing,Wang2022HowPA,sun2023feature}, they 
neglect
the interpretability that could arise mutually from the other domain. Specifically, while most spectral GNNs have well explained their learned filters in the spectral domain~\cite{he2021bernnet,Wang2022HowPA,he2022chebnetii,guo2023graph}, understandings from the spatial viewpoint are merely limited to fusing multi-scale graph information~\cite{liao2019lanczosnet}; this unfortunately lacks a deeper level of interpretability in the vertex domain. Therefore, a natural question arises: \textit{what information is essentially encoded by spectral GNNs in the spatial domain?}

In this work, we attempt to answer this question by exploring the connection between spectral filtering and spatial aggregation. The former is the key component in spectral GNNs, while the latter is closely associated with spatial GNNs utilizing 
recursive neighborhood aggregation.
In existing GNN frameworks, these two approaches rarely interact each other at the risk of domain information trade-offs due to uncertainty principles~\cite{heisenberg1927anschaulichen,folland1997uncertainty,agaskar2013spectral}. 
Recognizing the spatial significance in spectral filtering,~\cite{he2021bernnet} have recently considered non-negative constraints as part of a generalized graph optimization problem. Notably, however, spatial aggregation meanwhile resembles the optimizing trajectory of the same optimization problem through iterative steps, which may be easily overlooked.
Inspired by such observation, we examine, for the first time, the theoretical interaction between spectral filtering and spatial aggregation. This exploration has led us to uncover an intriguing theoretical interplay, \textit{i.e.}, spectral filtering implicitly modifies the original graph, transforming it into a new one that explicitly functions as a computation graph for spatial aggregation.
Delving deeper, we discover that the adapted new graph enjoys some desirable properties, enabling a direct link among nodes that originally require multiple hops to do so, thereby exhibiting nice non-locality. Moreover, we find that the new graph edges allow signed weights, which turns out capable of distinguishing between label agreement and disagreement of the connected nodes.

Overall, these findings underscore the interpretable role and significance of spectral GNNs in the spatial domain, inspiring us to rethink graph spectral filters beyond fixed-order polynomials, which, albeit efficient, limit models' effective propagation range and hinder their ability to capture long-range dependencies.
Concretely, we propose a novel Spatially Adaptive Filtering (SAF) framework to fully explore spectral GNNs in the spatial domain. SAF leverages the adapted new graph by spectral filtering for auxiliary spatial aggregation and allows individual nodes to flexibly balance between spectral and spatial features. By performing non-local aggregation with signed edge weights, 
SAF adeptly overcomes the limitations of truncated polynomials, enabling the model to capture both node similarity and dissimilarity at a global scale.
As a benefit, it can mitigate persistent deficiencies of GNNs regarding long-range dependencies and graph heterophily. The contributions are summarized as follows:

\begin{itemize}
    \item Our investigation into spectral GNNs in the spatial domain reveals that they fundamentally alters the original graph, introducing non-locality and signed edge weights to discern node label consistency.
    
    \item We propose Spatially Adaptive Filtering (SAF) framework, a paradigm-shifting approach to spectral GNNs that jointly leverages graph learning in both spatial and spectral domains, making it a powerful tool for capturing long-range dependencies and handling graph heterophily.

    \item Extensive experiments in node classification
    exhibit notable improvements of up to 15.37\%, and show that SAF beats the best-performing spectral GNNs on average.
\end{itemize}

\section{Motivation and Related Works}\label{sec:related_work}

This section outlines the motivation behind our research, derived from thoughtful consideration of the existing related works, and elucidates how our work diverges from and contributes to the current body of research.

\subsection{Graph Neural Networks}
GNNs can be broadly divided into spatial-based and spectral-based methods. 
Spatial GNNs leverage the spatial connections among nodes to perform message passing, also known as spatial aggregation~\cite{gilmer2017neural,hamilton2017inductive} (readers are directed to works~\cite{zhou2020graph,Wu2020ACS} for a thorough review). 
Spectral GNNs leverage the graph's spectral domain for convolution or, alternatively, 
spectral filtering~\cite{kipf2017semi,dong2021adagnn}.
Prevailing approaches focus on developing polynomial graph filters, by either learning polynomial coefficients, such as GPR-GNN~\cite{chien2021adaptive}, BernNet~\cite{he2021bernnet}, ChebNetII~\cite{he2022chebnetii}, and JacobiConv~\cite{Wang2022HowPA}, or optimizing the polynomial basis for better adaption, as seen in models like LON-GNN~\cite{tao2023longnn} and OptBasisGNN~\cite{guo2023OptBasisGNN}.
Diverging from this trend, ARMA~\cite{bianchi2021graph} employs rational filter functions 
while still approximating them with polynomials.
Although these methods are theoretically grounded in the spectral domain, their practical reliance on polynomial approximation hints at a profound linkage to the spatial domain. However, the spatial-domain interpretation of spectral GNNs is rarely examined. To this end, we delve into in this paper the intrinsic information spectral GNNs convey within the spatial context.

\subsection{Unified Viewpoints for GNNs.}

Several works have explored the nuances between spatial and spectral GNNs. Early studies by~\cite{balcilar2020bridging} and~\cite{chen2021bridging} examined their similarities in model formulations. \cite{chen2020simple} proved their spatial GNN's anti-oversmoothing ability via spectral analysis.
\cite{ma2021unified} and~\cite{zhu2021interpreting} utilized the graph signal denoising problem to integrate both GNN types, and their expressiveness equivalence is further explored 
in works~\cite{balcilar2021analyzing,Wang2022HowPA}.
Recently,~\cite{sun2023feature} have highlighted the feature space constraints of both spatial and spectral GNNs, 
while~\cite{guo2023clenshaw} tend to emphasize their relationship via residual connection.
Though these studies effectively bridge spectral and spatial GNNs, they remain focused on congruencies.
Unlike them, our work represents the first endeavor to delve into the interpretability of spectral GNNs in the spatial domain, emphasizing the theoretical synergy between spectral filtering and spatial aggregation.
The empirical success of our method 
(as compared to unified GNNs in Tables~\ref{tab:semi_nc} and~\ref{tab:full_nc}), 
stem from this in-depth analysis, further underscoring our practical contributions to the literature.

\subsection{Long-range Dependencies}

While substantial efforts have been directed towards capturing long-range dependencies in spatial GNNs~\cite{Klicpera2019PredictTP,chen2020simple,eliasof2021pde,pei2020geom,liu2021non,wu2022nodeformer,ding2024toward}, the exploration of the same challenge in spectral GNNs remains under-studied. 
Specifically, most spectral GNNs approximate graph filters with fixed-order polynomials,
which, albeit efficient, limits the effective propagation range and hinder their ability to capture the long-range dependencies.
To fill this gap, we propose SAF framework
that emerges as a valuable consequence of analyzing spectral GNNs in the spatial domain.
Aligned with our objective, Specformer~\cite{bo2023specformer} is introduced to addresses long-range dependencies for spectral GNNs, using a Transformer based set-to-set spectral filter. However, it lacks spatial-domain interpretability and introduce more trainable parameters. In contrast, our approach creates a non-local new graph without learning additional parameters, meanwhile elucidating the
spatial implications of spectral GNNs.
Similarly, a recent approach, FLODE~\cite{maskey2024fractional}, also produces matrices capable of capturing long-range dependencies by utilizing the fractional graph Laplacian~\cite{benzi2020non}.
However, this work mainly generalizes existing concepts from undirected to directed graphs to mitigate graph oversmoothing. Our study, conversely, focuses on exploring the fundamental issues surrounding undirected graphs to delve deeper into the intrinsic significance of spectral GNNs.

\subsection{Graph Heterophily}

Graph heterophily~\cite{pei2020geom,zhu2020beyond}, where different labeled nodes connect, challenges GNNs operating under the homophily assumption~\cite{mcpherson2001birds}. 
Although many GNNs have been crafted to manage heterophilic connections~\cite{fagcn2021,wang2021graph,chien2021adaptive,guo2022gnn,yan2023trainable,chen2023exploiting,yoo2023less,cui2023mgnn,li2024permutation},
our method offers a distinct solution. Specifically, SAF innovatively conducts an auxiliary non-local aggregation using signed edge weights, emphasizing both intra-class similarity and inter-class difference on a global scale.
One should note that there are two recent works~\cite{li2022finding,liang2023predicting} also employ signed edge weights, introducing GloGNN and LRGNN, respectively. GloGNN aims to capture global homophily but is limited to K-hop neighborhood information, while LRGNN extends this by using low-rank properties to approximate the true global relationships between node labels. Despite their success, both methods focus primarily on fitting node label relationships without fully exploring the fundamental GNN mechanisms, leading to a heavy reliance on label supervision (Section~\ref{apdix:sparse_advantange}).
Differently, our SAF leverages cross-domain insights into GNNs, ensuring the theoretical soundness of non-local learning, as proven in Section~\ref{sec:graphAnalysis}. This allows SAF to consistently perform well under both sparse and dense supervision.

\section{Notations and Preliminaries}

Let $\mathcal{G} = (\mathcal{V}, \mathcal{E})$ be a graph with node set $\mathcal{V}$ and edge set $\mathcal{E}$, where the number of nodes is denoted by $N$. 
The adjacency matrix $\mathbf{A} \in \mathbb{R}^{N \times N}$ defines the edge weights $A_{i,j}$ between nodes $v_i$ and $v_j$, and the degree matrix $\mathbf{D}$ can be obtained by summing the rows of $\mathbf{A}$ into a diagonal matrix.
We denote the graph Laplacian matrix as 
$\mathbf{L}$,
which is often normalized into $\hat{\mathbf{L}}= \mathbf{I} - \hat{\mathbf{A}}$ with $\hat{\mathbf{A}}=\mathbf{D}^{-\frac{1}{2}}\mathbf{A}\mathbf{D}^{-\frac{1}{2}}$ 
and an identity matrix $\mathbf{I}$. 
Let $\hat{\mathbf{L}}=\mathbf{U} \mathbf{\Lambda} \mathbf{U}^T$ be 
the Laplacian eigendecomposition,
where $\mathbf{U}$ is eigenvector matrix 
and $\mathbf{\Lambda}=\mathrm{diag}(\lambda_1,\lambda_2,\cdots,\lambda_N)$ consists of eigenvalues 
within $\left[0,2\right]$. 
For node classification, we usually have a feature matrix $\mathbf{X} \in \mathbb{R}^{N \times F}$ with $F$ being raw feature dimensions, and each node is assigned a
one-hot label vector $\mathbf{y}_i \in \mathbb{R}^C$ where $C \leq N$ is class number.

\subsection{Spectral Filtering}
Spectral filtering is essential in spectral GNNs. It selectively shrinks or amplifies the Fourier coefficients of node features~\cite{defferrard2016convolutional} for learning tasks and usually take the form as
\begin{equation}\label{eq:spec_filter}
    \mathbf{Z} = g_\psi(\hat{\mathbf{L}}) \mathbf{X} = \mathbf{U} g_\psi(\mathbf{\Lambda}) \mathbf{U}^T \mathbf{X}.
    % = \mathbf{U} \sum^K_{k=0} \psi_k P_K(\mathbf{\Lambda}) \mathbf{U}^T \mathbf{X}
\end{equation}
Here, $g_\psi: \left[0, 2\right] \to \mathbb{R}$ defines a graph filter function, which are often approximated by a $K$-order polynomial in practice. Specifically, we have $g_\psi(\lambda)=\sum^K_{k=0} \psi_k P_k(\lambda)=\sum^K_{k=0} \omega_k \lambda^k$ where $P_k: \left[0, 2\right] \to \mathbb{R}$ refers to a polynomial basis and both $\psi_k$ and $\omega_k$ denote the polynomial coefficient.

\subsection{Spatial Aggregation}
Spatial aggregation is a central component of spatial GNNs, facilitating the propagation of node information along edges and its subsequent aggregation within node neighborhood. To provide a better illustration, let's consider a popular spatial GNN, APPNP~\cite{Klicpera2019PredictTP}.
This model begins with a feature transformation -- $\mathbf{Z}^{(0)} = f(\mathbf{X})$, and then perform the spatial aggregation as:
\begin{equation}\label{eq:spat_agg}
\mathbf{Z}^{(k)} = (1-\eta) \mathbf{Z}^{(0)} + \eta \hat{\tilde{\mathbf{A}}} \mathbf{Z}^{(k-1)}, \quad k=1,2,\cdots,K,
\end{equation}
where $\hat{\tilde{\mathbf{A}}}=\tilde{\mathbf{D}}^{-\frac{1}{2}} \tilde{\mathbf{A}} \tilde{\mathbf{D}}^{-\frac{1}{2}}$, $\tilde{\mathbf{A}} = \mathbf{A} + \mathbf{I}$, and $\eta$ refers to the update rate.

\section{Rethinking Spectral Graph Neural Networks from the Spatial Perspective}
In this section, we provide both theoretical and empirical analyses to examine spectral GNNs from the spatial perspective and answer the question, \textit{i.e.}, what information is essentially encoded by spectral GNNs in the spatial domain?

\subsection{Interplay of Spectral and Spatial Domains through the Lens of Graph Optimization}\label{sec:interplay}
The graph signal denoising problem~\cite{shuman2013emerging} was initially leveraged in~\cite{ma2021unified,zhu2021interpreting} as a means to interpret GNNs with smoothness assumption, which yet does not always hold in certain real-world graph scenarios such as heterophily~\cite{zhu2020beyond}. Without loss of generality, 
in this work, we consider a more generalized graph optimization problem\footnote{This problem was first introduced in~\cite{he2021bernnet} for theoretically grounded graph filters. However, in this study, we repurpose it as a bridge between spectral filtering and spatial aggregation.}
\begin{equation}\label{eq:generalized_opt}
\mathop{\arg \min }_{\mathbf{Z}}\ \mathcal{L} = 
\alpha \|\mathbf{X} - \mathbf{Z}\|^2_F + (1 - \alpha) \cdot \mathrm{tr}(\mathbf{Z}^T
\gamma_{\theta}(\hat{\mathbf{L}}) \mathbf{Z})
\end{equation}
where $\mathbf{Z} \in \mathbb{R}^{N \times d}$ refers to node representations, $\gamma_\theta(\hat{\mathbf{L}})$ determines the rate of propagation~\cite{spielman2012spectral} by operating on the graph spectrum, i.e., $\gamma_{\theta}(\hat{\mathbf{L}}) = \mathbf{U} \gamma_{\theta}({\mathbf{\Lambda}}) \mathbf{U}^T$, and $\alpha \in (0, 1)$ is a trade-off coefficient.
In case of setting $\gamma_{\theta}(\hat{\mathbf{L}}) = \hat{\mathbf{L}}$, Eq.~(\ref{eq:generalized_opt}) turns into the well-known graph signal denoising problem.
To ensure the convexity of the objective in Eq.~(\ref{eq:generalized_opt}), a positive semi-definite constraint is imposed on $\gamma_{\theta}(\hat{\mathbf{L}})$, \textit{i.e.}, $\gamma_{\theta}(\lambda) \geq 0$ for $\lambda \in \left[0,2\right]$. 
Then, one can address this minimization problem through either closed-form or iterative solutions.

\subsubsection{Closed-form Solution} 
The closed-form solution can be obtained by setting the derivative of the objective function $\mathcal{L}$ to 0, i.e., $\frac{\partial \mathcal{L}}{\partial \mathbf{Z}} = 2 \alpha (\mathbf{Z} - \mathbf{X}) + 2 (1-\alpha) \gamma_\theta (\hat{\mathbf{L}}) \mathbf{Z} = 0 $.
Let $g_\psi (\lambda) = (1 + \frac{1-\alpha}{\alpha} \gamma_{\theta} (\mathbf{\lambda}))^{-1}$, 
we can observe that the closed-form solution in Eq.~(\ref{eq:clos_sol}) is equivalent to the spectral filtering in Eq.~(\ref{eq:spec_filter}).
\begin{equation}\label{eq:clos_sol}
\mathbf{Z}^{*} 
= (\mathbf{I} + \frac{1-\alpha}{\alpha} \gamma_{\theta} (\hat{\mathbf{L}}))^{-1} \mathbf{X}
=  g_\psi (\hat{\mathbf{L}}) \mathbf{X}
\end{equation}
As $\gamma_\theta (\lambda) \geq 0$, this establishes a more stringent constraint for the graph filter in spectral GNNs, i.e.,
$0 
<
g_\psi (\lambda) \leq 
\frac{\alpha}{\alpha + (1-\alpha) \cdot 0} =1$, which is termed as a non-negative constraint in this paper.

\subsubsection{Iterative Solution} 
Alternatively, we can take an iterative gradient descent method such that 
$\mathbf{Z}^{(k)} = \mathbf{Z}^{(k-1)} - b \frac{\partial \mathcal{L}}{\partial \mathbf{Z}}|_{\mathbf{Z}=\mathbf{Z}^{(k-1)}}$ with a step size $b=\frac{1}{2}$, which yields a concise iterative solution in Eq.~(\ref{eq:iter_sol}) with $\hat{\mathbf{A}}^{\text{new}} = \mathbf{I}-\gamma_\theta(\hat{\mathbf{L}})$. Notably, by taking $\hat{\mathbf{A}}^{\text{new}}$ as a new computation graph, this solution closely mirrors the spatial aggregation in Eq.~(\ref{eq:spat_agg}).
\begin{align}\label{eq:iter_sol}
\mathbf{Z}^{(k)} 
= \alpha \mathbf{X} + (1-\alpha) \hat{\mathbf{A}}^{\text{new}} \mathbf{Z}^{(k-1)}, \quad k=1,2,\cdots,K
\end{align}

\subsubsection{Theoretical Interaction --- the Adapted New Graph}
With the non-negative constraint, it is evident that both spectral filtering and spatial aggregation effectively address the generalized graph optimization problem in Eq.~(\ref{eq:generalized_opt}), despite their distinctive forms and operation domains.
Upon closer examination, we discover a compelling relationship between the graph filter $g_\psi(\lambda)$ in Eq.~(\ref{eq:clos_sol}) and the new graph $\hat{\mathbf{A}}^\text{new}$ in Eq.~(\ref{eq:iter_sol}),
given $g_\psi (\lambda) = (1 + \frac{1-\alpha}{\alpha} \gamma_{\theta} (\mathbf{\lambda}))^{-1}$,
\begin{equation}\label{eq:theo_interplay}
\hat{\mathbf{A}}^{\text{new}} 
= \mathbf{I}-\gamma_\theta(\hat{\mathbf{L}}) 
= \mathbf{I} - \frac{\alpha}{1-\alpha}(g_\psi(\hat{\mathbf{L}})^{-1} - \mathbf{I})
\end{equation}
which unveils an intrinsic inter-play, \textit{i.e.}, 
spectral filtering implicitly leads the original graph to an adapted new graph, explicitly computed for spatial aggregation.

\subsubsection{What are the differences between $\hat{\mathbf{A}}^\text{new}$ and $g_\psi(\hat{\mathbf{L}})$} Whereas the former as the uncovered new graph elucidates the inherent spatial node relationships, the latter is a operation in spectral GNNs that processes graph features within spectral domain.
It is crucial to understand that $g_\psi(\hat{\mathbf{L}})$ may not result in a dense matrix, especially with fixed-order polynomial approximation. 
This is because it captures up to only a $K$-hop neighborhood, i.e., $g_\psi(\hat{\mathbf{L}}) = \sum_{k=0}^K \psi_k P_k(\hat{\mathbf{L}}) = \sum_{k=0}^K \omega_k \hat{\mathbf{A}}^k$, practically limiting spectral GNNs' effective propagation range.
In contrast, 
our newfound graph 
$\hat{\mathbf{A}}^\text{new}$ intrinsically enjoys a non-local property, as confirmed in the following section.
Building upon this discovery, we further devise a framework to transcend domain barriers, overcoming the limitations of current spectral GNNs due to truncated polynomials (see details in Section~\ref{sec:saf}).

\subsection{In-depth Analysis of the Adapted New Graph}\label{sec:graphAnalysis}
To deepen our understanding of the interpretability produced by spectral GNNs in the spatial domain, we embark upon a blend of theoretical and empirical inquiries into the adapted new graph.

\begin{figure*}[t]
\centering
\includegraphics[width=0.98\textwidth]{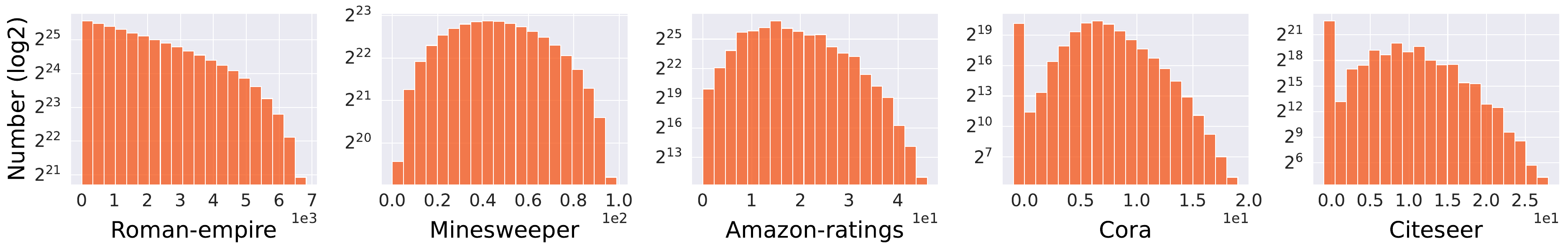}
\caption{
Distributions of connected nodes in the new graph based on their geodesic/shortest-path distance (as $\Delta_{i,j}$) in the original graph. Nodes, distant in the original graph ($\Delta_{i,j} > 1$ in x-axis), can be linked in the new graph (Number $> 0$ in y-axis).
}
\label{fig:nonLoc}
\end{figure*}

\begin{figure*}[t]
\centering
\includegraphics[width=0.98\textwidth]{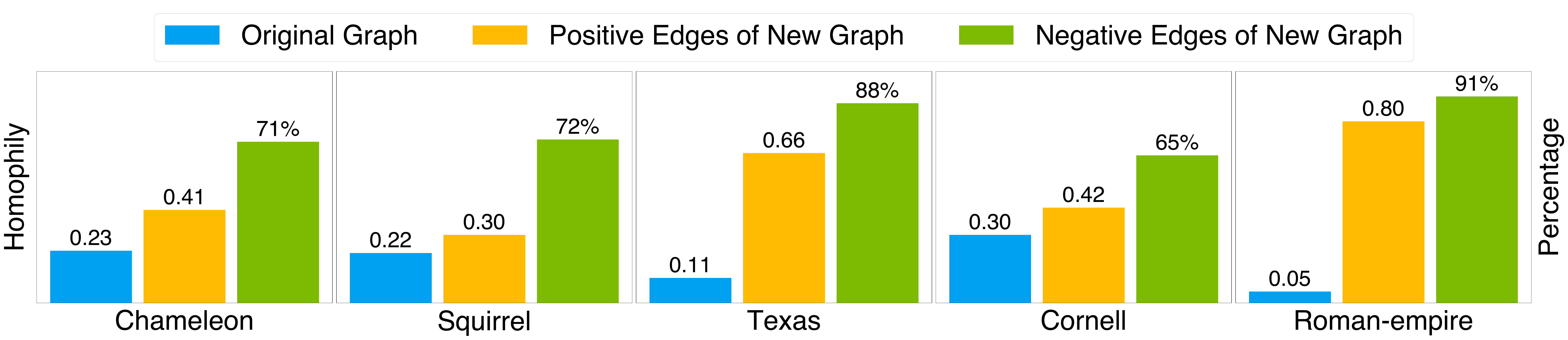}
\caption{Left y-axis: Homophily comparison between original and new graphs, considering only positive edges (blue and yellow bars). Right y-axis: 
Percentage of negative edges in the new graph that connect nodes from different classes (green bar).
}
\label{fig:discern_labelConsis}
\end{figure*}

\subsubsection{Non-locality}
Our examination of the adapted new graph illuminates its non-local nature, particularly evident in the infinite series expansion of the original graph's adjacency matrix. To elucidate, we first introduce an pivotal mathematical construct, the Neumann series, in the following lemma.
\begin{lemma}\label{lem:Neumann}
    Let $\mathbf{M} \in \mathbb{R}^{N \times N}$ be a matrix with eigenvalues $\lambda_n$, if $|\lambda_n| < 1$ for all $n=1,2, ..., N$, then $(\mathbf{I} - \mathbf{M})^{-1}$ exists and can be expanded as an infinite series, i.e., $(\mathbf{I} - \mathbf{M})^{-1} = \sum_{t=0}^\infty \mathbf{M}^t$, which is known as Neumann series.
\end{lemma}

\noindent With the established non-negative constraint on graph filters, specifically $0< g_\psi(\lambda) \leq 1$, it becomes evident that the eigenvalues of $\mathbf{I} - g_\psi(\hat{\mathbf{L}})$ falls into the interval permitting Neumann series expansion, as shown in lemma~\ref{lem:Neumann}~\cite{horn2012matrix}. Building on this observation, we present a non-trivial property of the new graph in the following proposition.

\begin{proposition}\label{prop:nonloc}
    Given adjacency matrix $\hat{\mathbf{A}}^{\text{new}}$ formulated in Eq.~(\ref{eq:theo_interplay}), the adapted new graph exhibits non-locality. Specifically, $\hat{\mathbf{A}}^{\text{new}}$ is expressible as an infinite series expansion of the original graph's adjacency matrix $\hat{\mathbf{A}}$. Formally, we have
    $
    \hat{\mathbf{A}}^{\text{new}} 
    = \mathbf{I} - \frac{\alpha}{1-\alpha} \sum_{t=1}^\infty (\mathbf{I} - 
\sum^K_{k=0} \pi_k \hat{\mathbf{A}}^k)^t
    = \sum_{t=0}^\infty \phi_t \hat{\mathbf{A}}^t
    $
    where $\pi_k$ and $\phi_t$ refer to the constant coefficients computed from 
    $\{\psi_0,\psi_1,...,\psi_K\}$ in distinct ways.
\end{proposition}

\begin{proof}
    We begin with the assertion that the eigenvalues of $\mathbf{I} - g_\psi(\hat{\mathbf{L}})$ are positive and strictly less than 1, which fulfills the necessary condition for the Neumann series expansion stated in Lemma~\ref{lem:Neumann}. As such, we can deduce $g_\psi(\hat{\mathbf{L}})^{-1} = (\mathbf{I} - (\mathbf{I} - g_\psi(\hat{\mathbf{L}})))^{-1} = \sum_{t=0}^\infty (\mathbf{I} - g_\psi(\hat{\mathbf{L}}))^t$. Owning to the prevalent polynomial approximation, we are eligible to express $g_\psi (\hat{\mathbf{L}})$ w.r.t. adjacency matrix $\hat{\mathbf{A}}$, i.e., $g_\psi (\hat{\mathbf{L}}) = g_\psi (\mathbf{I} - \hat{\mathbf{A}}) = \sum^K_{k=0} \pi_k \hat{\mathbf{A}}^k$ where $\pi_k$ refers to the new coefficients made of up $\{\psi_m\}^K_{m=0}$. Substituting this polynomial representation into our Neumann expansion, we obtain $g_\psi(\hat{\mathbf{L}})^{-1} = \sum_{n=0}^\infty (\mathbf{I} - \sum^K_{k=0} \pi_k \hat{\mathbf{A}}^k)^t$. Now, revisiting  $\hat{\mathbf{A}}^{\text{new}}$ in Eq.~(\ref{eq:theo_interplay}), we have 
    $\hat{\mathbf{A}}^{\text{new}} 
    = \mathbf{I} - \frac{\alpha}{1-\alpha}(g_\psi(\hat{\mathbf{L}})^{-1} - \mathbf{I})
    = \mathbf{I} - \frac{\alpha}{1-\alpha} \sum_{t=1}^\infty (\mathbf{I} - 
\sum^K_{k=0} \pi_k \hat{\mathbf{A}}^k)^t
= \sum_{t=0}^\infty \phi_t \hat{\mathbf{A}}^t$ where $\phi_t$ is a constant coefficient made up of $\{\pi_m\}^K_{m=0}$. 
\end{proof}

\noindent This proposition implies that the new graph engenders immediate links between nodes that originally necessitate multiple hops for connection. To further underpin this theoretical claim, we analyze the general connection status on the new graph by BernNet~\cite{he2021bernnet}, a spectral GNN adhering to the non-negative constraint. From Fig.~\ref{fig:nonLoc}, it is apparent that nodes originally separated by multiple hops achieve direct connections in the new graph.

\subsubsection{Signed Edge Weights --- Discerning Label Consistency}
Upon further scrutinizing the adapted new graph, we make a notable discovery that it readily accommodates both positive and negative edge weights. A more granular analysis in Fig.~\ref{fig:discern_labelConsis} reveals that a considerable portion of positive edge weights are assigned to the same-class node pairs, enhancing graph homophily (exemplified by edge homophily ratio~\cite{zhu2020beyond}). Conversely, edges parameterized with negative weights tend to bridge nodes with different labels. These findings demonstrate the newfound graph's adeptness in discerning label consistency among nodes. To theoretically explain this phenomenon, we further present the proposition below:
\begin{proposition}\label{prop:signedEdges}
    Let $\mathbf{Z}^*$ be the node representations optimized by Eq.~(\ref{eq:generalized_opt}). For $\mathbf{Z}^*$ to be effective in label prediction, it is a necessary condition that 
    $\hat{\mathbf{A}}^{\text{new}}$ accommodates both positive and negative edge weights s.t. for any node pairs $v_i, v_j \in \mathcal{V}$,
    $\hat{A}^{\text{new}}_{i,j} > 0$ if $\mathbf{y}_i=\mathbf{y}_j$ and $\hat{A}^{\text{new}}_{i,j} < 0$ if $\mathbf{y}_i \neq \mathbf{y}_j$.
\end{proposition}

\begin{proof}
Let us commence the proof by contradiction. Let $\mathcal{C}$ denote the condition described in proposition~\ref{prop:signedEdges}. Assume, for the sake of contradiction, that $\mathcal{C}$ is not requisite for the optimal node representations $\mathbf{Z}^*$ to be predictive of node labels. Under this assumption, there are node pairs $v_i, v_j \in \mathcal{V}$ such that: (1) if $\mathbf{y}_i = \mathbf{y}_j$, $\hat{A}^{\text{new}}_{i,j} < 0$; (2) if $\mathbf{y}_i \neq \mathbf{y}_j$, $\hat{A}^{\text{new}}_{i,j} > 0$.
Without loss of generality, given the non-locality as proved in proposition~\ref{prop:nonloc}, we exclude cases where $\hat{A}^{\text{new}}_{i,j} = 0$ from our consideration. 
Now, consider the second objective term $\mathrm{tr}(\mathbf{Z}^T \gamma_{\theta}(\hat{\mathbf{L}}) \mathbf{Z})$ in Eq.~(\ref{eq:generalized_opt}). Using the relationship $\gamma_\theta(\hat{\mathbf{L}}) = \mathbf{I} - \hat{\mathbf{A}}^{\text{new}}$, 
we can expand this term into
$
\sum_{v_i,v_j \in \mathcal{V}} \hat{A}^{\text{new}}_{i,j} \|\mathbf{Z}_i - \mathbf{Z}_j\|^2_2 
$.
Under (1), for same-class nodes $v_i, v_j$ with $\hat{A}^{\text{new}}_{i,j} < 0$, minimizing the objective term pulls $\mathbf{Z}_i$ and $\mathbf{Z}_j$ apart in the latent space. This behavior violates the canonical understanding that nodes from the same class should exhibit similar representations. Under (2), for different-class nodes $v_i, v_j$ with $\hat{A}^{\text{new}}_{i,j} > 0$, the optimization encourages $\mathbf{Z}_i$ and $\mathbf{Z}_j$ to be more similar. This is in direct opposition to the basic classification principle that nodes from different classes should have distinct representations. Given these contradictions stemming from the mathematical implications in optimization, we must reject assumptions (1) and (2), affirming the necessary condition $\mathcal{C}$ 
for accurate label prediction by $\mathbf{Z}^*$.
\end{proof}

\noindent Proposition~\ref{prop:signedEdges} provides a theoretical foundation of our empirical findings on the new graph. 
The essence lies in the objective in Eq.~(\ref{eq:generalized_opt}), particularly the trace term $\mathrm{tr}(\mathbf{Z}^T\gamma_{\theta}(\hat{\mathbf{L}}) \mathbf{Z})$.
For clarity, let us reinterpret this trace term 
as $\mathrm{tr}(\bar{\mathbf{Z}}^T (\mathbf{D}^{\text{new}} - \mathbf{A}^{\text{new}}) \bar{\mathbf{Z}})$, where $\mathbf{D}^{\text{new}}$ denotes the related degree matrix and $\bar{\mathbf{Z}}$ is derived from rescaling $\mathbf{Z}$. 
Clearly, this term evaluates label smoothness among adjacent nodes in the new graph, which, given its non-local nature, includes both intra-class ($=$) and inter-class ($\neq$) node connections such that $\mathbf{A}^{\text{new}} = \mathbf{A}^{\text{new}}_{=} + \mathbf{A}^{\text{new}}_{\neq}$.
Drawing from proposition~\ref{prop:signedEdges}, we can further dissect the original trace term, splitting it into $\mathrm{tr}(\bar{\mathbf{Z}}^T (\mathbf{D}^{\text{new}}_{=} - \mathbf{A}^{\text{new}}_{=}) \bar{\mathbf{Z}}) - \mathrm{tr}(\bar{\mathbf{Z}}^T | (\mathbf{D}^{\text{new}}_{\neq} -\mathbf{A}^{\text{new}}_{\neq}) | \bar{\mathbf{Z}})$ where the $|\cdot|$ operation denotes absolute values. 
As such, it becomes evident that minimizing this trace term not only enhances the representational proximity for same-class node pairs but also strengthens the distinctiveness for different-class nodes pairs. 
Such nuanced behaviors, inherent to the optimization in Eq.~(\ref{eq:generalized_opt}), are necessary for GNN models to achieve accurate label predictions.

\subsubsection{Summarization}
Our investigation into spectral GNNs in the spatial domain reveals that 
they
fundamentally alters the original graph, 
imbuing it with non-locality and signed edge weights that capture label relationships among nodes.
These findings highlight the interpretable role of spectral GNNs in the spatial domain and prompt us to rethink current models beyond the truncated polynomial filters.

\begin{figure}[t]
\centering
\includegraphics[clip, trim=0cm 6.6cm 5.7cm 0cm,width=0.48\textwidth]{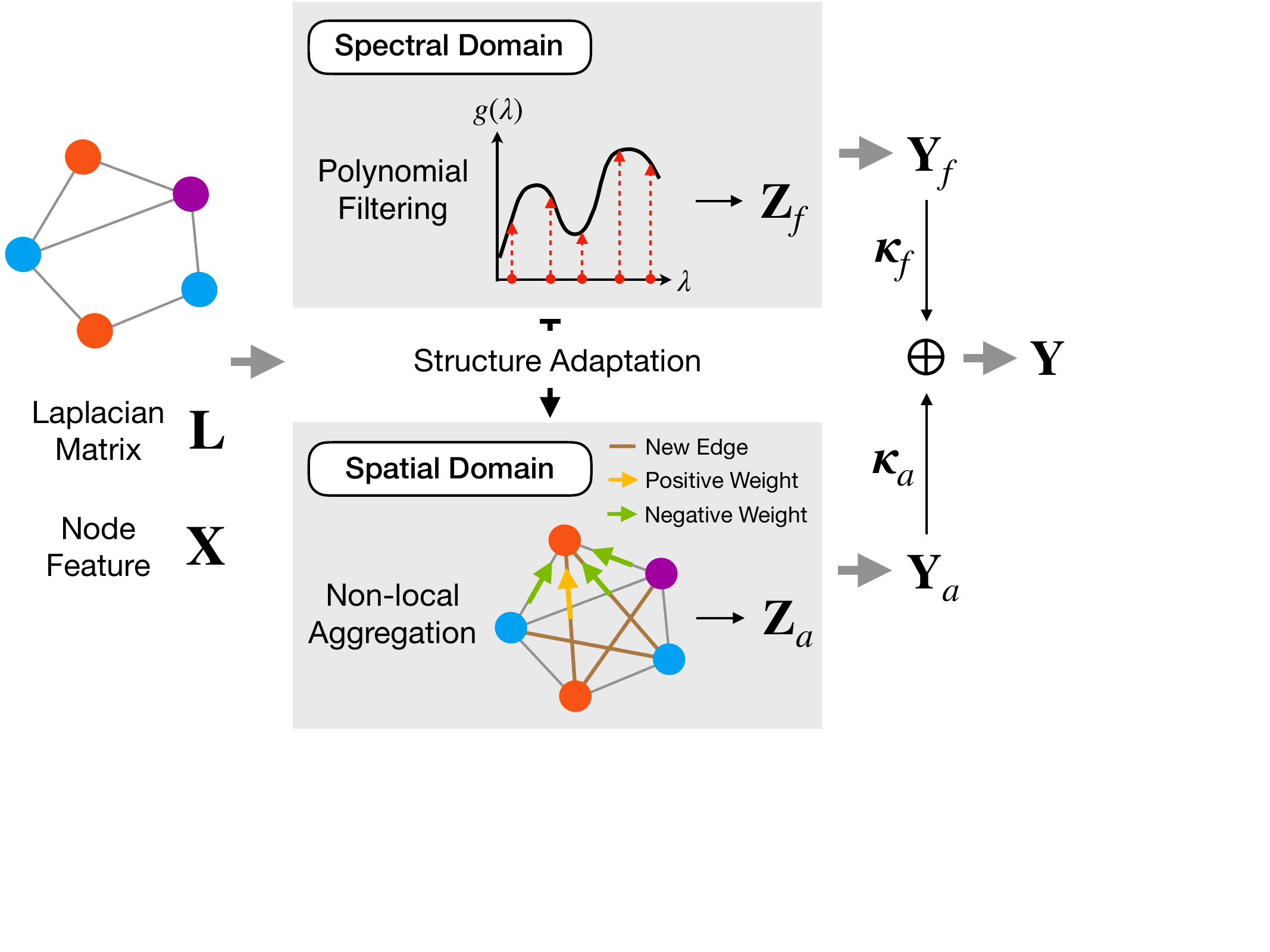}
\caption{Illustration of SAF framework where varying node colors represent different node labels. SAF leverages the adapted new graph by spectral filtering for auxiliary non-local aggregation in the spatial domain and allows individual nodes to flexibly balance between spectral and spatial features. 
}
\label{fig:saf_pipeline}
\end{figure}

\section{Spatially Adaptive Filtering Framework}\label{sec:saf}
Building on our discoveries, we re-evaluate the state-of-the-art spectral GNNs and put forth a paradigm-shifting framework, Spatially Adaptive Filtering (SAF), for joint exploitation of graph-structured data across both spectral and spatial domains.
SAF leverages the adapted new graph by spectral filtering for an auxiliary non-local aggregation, addressing enduring challenges in GNNs related to long-range dependencies and graph heterophily. See the overall pipeline in Fig.~\ref{fig:saf_pipeline}.

\subsubsection{Non-negative Spectral Filtering}
The proposed SAF requires explicit computation of the newfound graph, as outlined in Eq.~(\ref{eq:theo_interplay}). 
This further necessitates the graph filter \(g_\psi: \left[0, 2\right] \to \mathbb{R}\) to satisfy the non-negative constraint from Eq.~(\ref{eq:generalized_opt}): $0 \leq g_\psi (\lambda) \leq 1$.
However, not all extant graph filters fulfill this prerequisite. For instance, the filter use by GCN~\cite{kipf2017semi}, $g_\psi (\lambda) = 1 - \lambda$, takes negative values when $\lambda > 1$. 
In this research, we approximate the graph filter using Bernstein polynomials~\cite{farouki2012bernstein}, which are known for their non-negative traits~\cite{powers2000polynomials} and are essential in a preeminent spectral GNN, BernNet~\cite{he2021bernnet}.
For $g_\psi (\lambda) \leq 1$ part, we rescale Bernstein polynomials with the following proposition.
\begin{proposition}\label{prop:max_bern}
Let $B_{k,K}(x)$ denote the Bernstein polynomial basis of index $k$ and order $K$, which is defined as $B_{k,K}(x) = \binom{K}{k} (1-x)^{K-k} x^k$ for $x \in [0, 1]$. Let $\psi_k$ denote the $k$-th coefficient of a polynomial $p(x)$ of order $K$, where $p(x) = \sum_{k=0}^{K} \psi_k B_{k,K}(x)$ with $\psi_k \geq 0$ for all $k$. Then for all $x \in [0, 1]$, we have $g_\psi(x) \leq \max \{\psi_k\}_{k=0}^K$.
\end{proposition}

\begin{proof}
We denote $p(x) = \sum_{k=0}^K \psi_k \binom{K}{k} (1-x)^{K-k} x^k$ as a Bernstein polynomial with $\psi_k \geq 0$ for all $k$ and $\psi_\text{max}=\max \{\psi_k\}_{k=0}^K$. Given $x \in [0, 1]$, we can derive the following inequality as
\begin{align*}
p(x) &= \sum_{k=0}^K \psi_k \binom{K}{k} (1-x)^{K-k} x^k \\
&\leq \psi_\text{max} \sum_{k=0}^K  \binom{K}{k} (1-x)^{K-k} x^k \\
&= \psi_\text{max} (1-x+x)^K = \psi_\text{max}.
\end{align*}
Therefore, we have $p(x) \leq \max \{\psi_k\}_{k=0}^K$ for all $x \in [0, 1]$.
\end{proof}

\noindent Proposition~\ref{prop:max_bern} suggests that the Bernstein polynomial function attains its maximum value in $\psi_\text{max} = \max \{\psi_k\}_{k=0}^K$.
Therefore, $g_\psi(\lambda)$ can be rescaled within $[0, 1]$ by $\hat{g}_\psi(\lambda) = \frac{1}{\psi_\text{max}}\sum_{k=0}^K \psi_k B_{k,K}(\frac{\lambda}{2})$, which enables us to formulate the spectral filtering in SAF as
$\mathbf{Z}_f = \hat{g}_\psi(\hat{\mathbf{L}}) f_\varphi(\mathbf{X})
 = \frac{1}{\psi_\text{max}}
 \sum_{k=0}^K \psi_k \frac{1}{2^K} \binom{K}{k} (2\mathbf{I}-\hat{\mathbf{L}})^{K-k} \hat{\mathbf{L}}^k f_\varphi(\mathbf{X})$
where $f_\varphi(\cdot)$, a two-layer MLP, maps $\mathbf{X}$ from $F$ to $C$ dimensions using 64 hidden units, and $\{\psi_k\}_{k=1}^K$ are non-negative learnable parameters. 
Note that SAF also permits alternative implementations such as using Chebyshev polynomials~\cite{hammond2011wavelets,defferrard2016convolutional} for graph filter learning, enhancing models like ChebNetII~\cite{he2022chebnetii} (see details in Section~\ref{apdix:chebNetAsBase}).

\subsubsection{Non-local Spatial Aggregation}
Once acquiring a suitable spectral filter $\hat{g}_\psi(\lambda)$, we compute the adapted new graph as 
$\hat{\mathbf{A}}^{\text{new}} = \mathbf{I} - \tau (
\mathbf{U_m} g_\psi(\mathbf{\Lambda_m})^{-1} \mathbf{U_m}^T
- \mathbf{I})$ by Eq.~(\ref{eq:theo_interplay}) where $\tau = \frac{\alpha}{1-\alpha}$ is a scaling parameter and a partial eigendecomposition~\cite{lanczos1950iteration} can be employed to produce a low-rank, robust structure for $\hat{\mathbf{A}}^{\text{new}}$ using only $m$ extremal eigenvalues.
Equipped with this newfound graph, we proceed to perform non-local aggregation:
$\mathbf{Z}^{(l)} = (1-\eta) \mathbf{Z}^{(0)} + \eta  \hat{\mathbf{A}}^{\text{new}} \mathbf{Z}^{(l-1)}$, where $\eta$ refers to the update rate and $\mathbf{Z}^{(0)}=f_\varphi(\mathbf{X})$ and $l=1,2,\cdots,L$ denote layer number. 
The iteratively aggregated results are denoted as $\mathbf{Z}_a$. Recognizing the potential noise from the non-local nature of $\hat{\mathbf{A}}^{\text{new}}$, we apply a sparsification technique, leveraging a positive threshold $\epsilon$, and retain only essential elements outside the $\left[-\epsilon, \epsilon \right]$ interval. 
For clarity, this refined model is referred to as SAF-$\epsilon$.

\subsubsection{Node-wise Prediction Amalgamation}
To leverage information from different graph domains, 
we
employ an attention mechanism, allowing nodes to determine the importance of each space. 
This mechanism 
produces pairwise weights for a nuanced amalgamation during 
prediction.
Specifically, the weight pair is computed as
$\boldsymbol{\kappa}_f = \text{Sigmoid} (\mathcal{P}_f (\mathbf{Z}_f))$, $\boldsymbol{\kappa}_a = \text{Sigmoid} (\mathcal{P}_a (\mathbf{Z}_a))$
where $\boldsymbol{\kappa}_f, \boldsymbol{\kappa}_a \in \mathbb{R}^N$ contain the weights for each node, and $\mathcal{P}_f(\cdot)$ and $\mathcal{P}_a(\cdot)$ are two different mappings from $\mathbb{R}^C$ to $\mathbb{R}$. For simplicity, we implement them using two one-layer MLPs. Given domain predictions $\mathbf{Y}_f, \mathbf{Y}_a \in \mathbb{R}^C$, the final model prediction is attained as
$
\mathbf{Y} = \mathrm{diag}(\boldsymbol{\kappa}_f) \cdot \mathbf{Y}_f + \mathrm{diag}(\boldsymbol{\kappa}_a) \cdot \mathbf{Y}_a
$
where a normalization $\left[\boldsymbol{\kappa}_f, \boldsymbol{\kappa}_a\right] \gets \frac{\left[\boldsymbol{\kappa}_f, \boldsymbol{\kappa}_a\right]}{\max \{\| \{\boldsymbol{\kappa}_f, \boldsymbol{\kappa}_a\}\|_1, \delta\}}$ is performed beforehand to maintain $\boldsymbol{\kappa}_f + \boldsymbol{\kappa}_a = \mathbf{1}$ with small value $\delta$ preventing zero division. 
Similar schemes can be founded in works~\cite{sagi2018ensemble,wang2020gcn,zhu2020graph}.

\begin{table}[t]
\centering
\caption{Statistics of real-world datasets. $F$ and $C$ denotes the number of features and classes. $\Delta$ represents graph diameter referring to the longest geodesic distance between nodes on the graph; for Penn94, due to multiple subgraphs, we report $\Delta$ of the largest connected component. We assess graph homophily using three metrics - edge homophily~\cite{zhu2020beyond} $\mathcal{H}$, class homophily~\cite{lim2021large} $\mathcal{H}_\text{class}$, and adjusted homophily~\cite{platonov2023characterizing} $\mathcal{H}_\text{adjusted}$ - which ranges from 0 (high heterophily) to 1 (high homophily). 
}
\label{tab:data_sta}
\setlength\tabcolsep{3pt}
\resizebox{0.48\textwidth}{!}{
\begin{tabular}{lcccccccc}
\toprule
Dataset   &$|\mathcal{V}|$  &$|\mathcal{E}|$ & $F$ &$C$ 
& $\Delta$
& $\mathcal{H}$     & $\mathcal{H}_\text{class}$     & $\mathcal{H}_\text{adjusted}$ \\
\midrule
Chameleon &2,227       &36,101       &2,325          &5         &11   &0.23   & 0.06      &0.03   \\
Squirrel  &5,201       &217,073       &2,089          &5         &10   &0.22   & 0.03      &0.01   \\
Texas     &183       &309       &1,703          &5         &   8    &0.11   &    0.00   &-0.23   \\
Cornell   &183       &295       &1,703          &5         &   8    &0.30   &  0.05     &-0.08   \\
Actor   &7,600      &33,544      &931      &5      &  12 &0.22      &  0.01     &0.00      \\
Cora      &2,708       &5,429       &1,433          &7         &   19  &0.81   &   0.77    &0.77   \\
Citeseer  &3,327       &4,732       &3,703          &6         & 28  &0.74   &    0.63   &0.67   \\
Pubmed    &19,717       &44,338       &500       &3       & 18  &0.80      &0.66   &0.69   \\
\midrule
Minesweeper   &10,000      &39,402      &7      &2      & 99    &0.68      &   0.01    &0.01      \\
Tolokers   &11,758      &519,000      &10      &2      & 11  &0.59      &    0.18   &0.09      \\
Amazon-ratings   &24,492      &93,050      &300      &5      & 46&0.38      &    0.13   &0.14      \\
Roman-empire   &22,662      &32,927      &300      &18      &  6,824    &0.05      &  0.02     &-0.05      \\
Penn94   &41,554    &1,362,229  &5  &2  &8  &0.47   &0.05   &0.02    \\
\bottomrule
\end{tabular}
}
\end{table}

\subsubsection{Complexity Analysis}\label{sec:complexity}
SAF augments spectral GNNs with non-local aggregation and node-wise amalgamation. The first part entails creating a new graph and information propagation.
In SAF-$\epsilon$, these two steps are separated, culminating in $\mathcal{O}(N^3 + N^2 + nnz(\hat{\mathbf{A}}^{\text{new}})d)$ complexity,
where $nnz$ denotes non-zero element count. Conversely, SAF, viewing non-local aggregation holistically, can reduce complexity to $\mathcal{O}(2dN^2 + dN)$ when $d \ll N$. For node-wise amalgamation, its parallelizable nature ensures computational efficiency. Our method also requires eigendecomposition
precomputation, which, though naively complex at $\mathcal{O}(N^3)$, can be reduced to $\mathcal{O}(m^2 + nnz(\hat{\mathbf{L}})m)$ using Lanczos method~\cite{lanczos1950iteration} with $m \ll N$ iterative steps. The results are also reusable for both training and inference.
We present empirical studies on both time and space overheads in Section~\ref{apdix:run_time}.

\subsubsection{Practical Implications of Eigendecomposition}\label{apdix:spec_decomp}
Eigendecomposition breaks down a matrix into its eigenvalues and eigenvectors, offering insights into matrix properties, especially for the graph Laplacian. Despite computational demands, this technique has attracted surging interest in the graph learning community due to its theoretical richness, and it can be practically expedited for larger graphs using Lanczos~\cite{lanczos1950iteration} and Sparse Generalized Eigenvalue~\cite{cai2021note} algorithms. Recent advancements also underscore its value in various applications such as graph positional encoding~\cite{belkin2003laplacian,dwivedi2020benchmarking,wang2022equivariant}, spectral graph convolution~\cite{liao2019lanczosnet}, graph domain adaptation~\cite{you2022graph}, graph robustness~\cite{chang2021not}, graph expressivity~\cite{lim2022sign,lim2023expressive,yan2023trainable,sun2023feature}, and graph transformers~\cite{kreuzer2021rethinking,kim2022pure,rampavsek2022recipe,bo2023specformer}. 
In line with these developments, our SAF also utilizes eigendecomposition to explicitly create a new graph, enabling efficient non-local aggregation with signed weights to tackle long-range dependency and graph heterophily.

\begin{table*}[t]
\caption{Semi-supervised node classification accuracy (\%) $\pm$ 95\% confidence interval.}\label{tab:semi_nc}
\centering
\resizebox{0.98\textwidth}{!}{
\begin{tabular}{lcccccccc}
\toprule
\textbf{Method} & \textbf{Cham.} & \textbf{Squi.} & \textbf{Texas} & \textbf{Corn.} & \textbf{Actor} & \textbf{Cora} & \textbf{Cite.} & \textbf{Pubm.} \\
\midrule
MLP         & 26.36{$\pm$2.85}                             & 21.42{$\pm$1.50}                             & 32.42{$\pm$9.91}                             & 36.53{$\pm$7.92}                             & 29.75{$\pm$0.95}                             & 57.17{$\pm$1.34}                            & 56.75{$\pm$1.55}                             & 70.52{$\pm$2.01}                             \\
GCN         & 38.15{$\pm$3.77}                             & 31.18{$\pm$0.93}                             & 34.68{$\pm$9.07}                             & 32.36{$\pm$8.55}                             & 22.74{$\pm$2.37}                             & 79.19{$\pm$1.37}                            & 69.71{$\pm$1.32}                             & 78.81{$\pm$0.84}                             \\
APPNP & 32.73{$\pm$2.31} & 24.50{$\pm$0.89} & 34.79{$\pm$10.11} & 34.85{$\pm$9.71} & 29.74{$\pm$1.04} & 82.39{$\pm$0.68} & 69.79{$\pm$0.92} & \underline{79.97{$\pm$1.58}} \\
\midrule
ARMA	&37.42{$\pm$1.72}	&24.15{$\pm$0.93}	&39.65{$\pm$8.09}	&28.90{$\pm$10.07}	&27.02{$\pm$2.31}	&79.14{$\pm$1.07}	&69.35{$\pm$1.44}	   &78.31{$\pm$1.33}\\
GPR-GNN     & 33.03{$\pm$1.92}                             & 24.36{$\pm$1.52}                             & 33.98{$\pm$11.90}                            & 38.95{$\pm$12.36}                            & 28.58{$\pm$1.01}                             & 82.37{$\pm$0.91}                            & 69.22{$\pm$1.27}                             & 79.28{$\pm$2.25}                             \\
BernNet     & 27.32{$\pm$4.04}                             & 22.37{$\pm$0.98}                             & 43.01{$\pm$7.45}                             & 39.42{$\pm$9.59}                             & 29.87{$\pm$0.78}                             & 82.17{$\pm$0.86}                            & 69.44{$\pm$0.97}                             & 79.48{$\pm$1.47}                             \\
ChebNetII   & \textbf{43.42{$\pm$3.54}}                             & \textbf{33.96{$\pm$1.22}}                             & 46.58{$\pm$7.68}                             & 42.19{$\pm$11.61}                            & 30.18{$\pm$0.81}                             & 82.42{$\pm$0.64}                            & 69.89{$\pm$1.21}                             & 79.51{$\pm$1.03}                             \\
JacobiConv                              & 36.67{$\pm$1.63}                            & 29.38{$\pm$0.71}                            & 48.50{$\pm$5.90}                            & 43.01{$\pm$11.92}                           & 31.69{$\pm$0.71}                            & \underline{82.93{$\pm$0.55}}                           & 70.25{$\pm$1.02}                            & 79.53{$\pm$1.28}                            \\
Specformer  	&36.05{$\pm$3.47}	&29.64{$\pm$0.88}	&50.00{$\pm$8.33}	&\underline{43.76{$\pm$5.84}}	&31.45{$\pm$0.68}	&81.44{$\pm$0.63}	&66.11{$\pm$0.98}	&78.05{$\pm$1.03}\\

LON-GNN     & 35.17{$\pm$1.85} & 30.25{$\pm$1.04} & 45.38{$\pm$7.92} & 35.32{$\pm$8.09}  & 31.51{$\pm$1.23} & 81.93{$\pm$0.74} & \underline{70.41{$\pm$1.10}} & 79.57{$\pm$1.08} \\
OptBasisGNN & 35.56{$\pm$2.86} & 31.25{$\pm$1.06} & 37.11{$\pm$5.09} & 32.31{$\pm$7.11}  & 31.73{$\pm$0.50} & 78.69{$\pm$0.86} & 63.46{$\pm$1.30} & 77.38{$\pm$0.98} \\
FLODE	&40.20$\pm$1.02      &31.99$\pm$1.05      &\underline{50.29$\pm$4.74}      &42.89$\pm$7.69      &\underline{32.18$\pm$1.10}      &79.90$\pm$1.07      &69.89$\pm$2.03      &77.78$\pm$1.47      \\
\midrule
GNN-LF      & 26.49{$\pm$2.00} & 22.01{$\pm$1.04} & 39.02{$\pm$6.24} & 36.65{$\pm$9.60}  & 28.28{$\pm$0.71} & 81.96{$\pm$0.92} & 69.80{$\pm$1.36} & 79.50{$\pm$1.28} \\
GNN-HF      & 35.57{$\pm$2.26} & 22.36{$\pm$1.26} & 44.80{$\pm$5.67} & 38.79{$\pm$11.62} & 29.15{$\pm$0.78} & 81.15{$\pm$0.78} & 69.68{$\pm$0.73} & 79.10{$\pm$1.19} \\
ADA-UGNN    & 39.39{$\pm$2.02} & 25.65{$\pm$0.49} & 47.86{$\pm$6.65} & 42.89{$\pm$8.09}  & 30.78{$\pm$1.00} & 82.52{$\pm$1.04} & 70.18{$\pm$1.40} & 79.78{$\pm$1.32} \\
FE-GNN      & 38.23{$\pm$1.66} & 31.67{$\pm$1.60} & 47.40{$\pm$5.90} & 41.21{$\pm$8.96}  & 26.20{$\pm$0.76} & 77.00{$\pm$0.74} & 61.24{$\pm$1.26} & 75.63{$\pm$1.33} \\
\midrule
SAF	&41.82{$\pm$1.74}	&31.77{$\pm$0.69}	&58.04{$\pm$3.76}	&52.49{$\pm$8.56}	&33.50{$\pm$0.55}	&83.57{$\pm$0.66}	&71.07{$\pm$1.08}	&79.51{$\pm$1.12} \\
SAF-$\epsilon$    & \underline{41.88{$\pm$2.04}}                            & \underline{32.05{$\pm$0.40}}                            & \textbf{58.38{$\pm$3.47}}                            & \textbf{53.41{$\pm$5.55}}                            & \textbf{33.84{$\pm$0.58}}                            & \textbf{83.79{$\pm$0.71}}                           & \textbf{71.30{$\pm$0.93}}                            &             
\textbf{80.16{$\pm$1.25}}
\\
Improv.\footnotemark[3]     & 14.56\%    & 9.68\%     & 15.37\%    & 13.99\%    & 3.97\%     & 1.62\%    & 1.86\%     &0.68\%\\   \bottomrule                          
\end{tabular}}
\end{table*}

\section{Experiments}\label{sec:exp}

\subsection{Datasets and Experimental Setup}

\subsubsection{Datasets}

We conduct experiments on 13 real-world datasets with detailed statistics provided in Table~\ref{tab:data_sta}. In certain compact sections of this paper, we use four-letter abbreviations for dataset names. 
\begin{itemize}
    \item Chameleon and Squirrel are Wikipedia networks collected by~\cite{rozemberczki2021multi} where nodes are web pages connected by mutual links. We utilize the labels from~\cite{pei2020geom}, dividing nodes into five classes based on their average monthly traffic.
    \item Texas and Cornell are webpage datasets from WebKB project, containing nodes from five classes (student, project, course, staff, and faculty) and connected by hyperlinks. This study uses the preprocessed version by~\cite{pei2020geom}.
    \item Actor is an actor co-occurrence network induced from the film-director-actor-writer network~\cite{tang2009social} by~\cite{pei2020geom}. In this network, nodes are actors divided into five classes, and edges represent actor co-occurrence on Wikipedia pages. 
    \item Cora, Citeseer, and Pubmed are citation network benchmarks~\cite{sen2008collective,namata2012query} with nodes as scientific papers and edges as undirected citations. Each node is assigned with a class label by research topic and bag-of-word features.
    \item Minesweeper, Tolokers, Amazon-ratings, and Roman-empire are recently proposed benchmarks by~\cite{platonov2023critical} for specifically evaluating GNN models under heterophily.
    \item Penn94 is a large-scale friendship network~\cite{traud2012social} from the Facebook 100 networks, where nodes denote students labeled by reported genders and posses features such as major, second major/minor, dorm/house, year, and high school. This work uses the version preprocessed by~\cite{lim2021large}.
\end{itemize}

\subsubsection{Baselines}
We compare SAF with 21 models:
(1) MLP; 
(2) Basic GNNs: GCN~\cite{kipf2017semi} and APPNP~\cite{Klicpera2019PredictTP}; 
(3) Spectral GNNs: ARMA~\cite{bianchi2021graph}, GPR-GNN~\cite{chien2021adaptive}, BernNet~\cite{he2021bernnet}, ChebNetII~\cite{he2022chebnetii}, JacobiConv~\cite{Wang2022HowPA}, Specformer~\cite{bo2023specformer}, LON-GNN~\cite{tao2023longnn}, OptBasisGNN~\cite{guo2023OptBasisGNN} and FLODE~\cite{maskey2024fractional}; 
(4) Spatial GNNs: GCNII~\cite{chen2020simple}, PDE-GCN~\cite{eliasof2021pde}, NodeFormer~\cite{wu2022nodeformer}, GloGNN++~\cite{li2022finding} and LRGNN~\cite{liang2024predicting}; 
(5) Unified GNNs: GNN-LF~\cite{zhu2021interpreting}, GNN-HF~\cite{zhu2021interpreting}, 
ADA-UGNN~\cite{ma2021unified} and FE-GNN~\cite{sun2023feature}.

\subsubsection{Implementation Details}
To follow~\cite{he2021bernnet,he2022chebnetii,Wang2022HowPA}, we fix $K = 10$. For each dataset, we perform a grid search to tune the hyper-parameters of all models. With the best hyper-parameters, we train models with Adam optimizer~\cite{kingma2014adam} in 1,000 epochs using early-stopping strategy and a patience of 200 epochs, and report the mean classification accuracies with a 95\% confidence interval on 10 random data splits. 
As~\cite{he2022chebnetii} have made a thorough evaluation and share the same experimental protocol with us, we leverage their results for models: MLP, GCN, APPNP, ARMA, GPR-GNN, BernNet, ChebNetII, GCNII and PDE-GCN on datasets -- Chameleon, Squirrel, Texas, Cornell, Actor, Cora, Citeseer, and Pubmed. For JacobiConv, LON-GNN, and OptBasisGNN, we also report the results from their papers~\cite{Wang2022HowPA,tao2023longnn,guo2023OptBasisGNN}. 
Our codes will be made available if the paper could be accepted.

\subsubsection{Hyper-parameters Setting}\label{apidx:param_searching}
We perform a grid search on the hyper-parameters of all models for each dataset using Optuna~\cite{akiba2019optuna}
To accommodate extensive experiments across diverse datasets in both semi- and full-supervised setting, we define a broad searching space as: learning rate $\text{lr} \sim $ \{1e-3, 5e-3, 1e-2, 5e-2, 0.1\}, weight decay $\text{L}_2 \sim$ \{0.0, 1e-6, 5e-6, 1e-5, 5e-5, 1e-4, 5e-4, 1e-3, 5e-3, 1e-2\}, dropout $\sim$ \{0.0, 0.1, ..., 0.8\} with step 0.1, non-local aggregation step $L \sim $ \{1,2, ...,10\} with step 1, scaling parameter $\tau \sim $ \{0.1, 0.2, ..., 1.0\} with step 0.1,
update rate $\eta \sim \ $ \{0.1, 0.2, ..., 1.0\} with step 0.1, and threshold $\epsilon \sim \ $\{0.0, 1e-5, 5e-5, 1e-4, 5e-4, 1e-3, 5e-3, 1e-2\}. For other parameters specific to different base models, we strictly follow their instructions in the original papers. 

\begin{table*}[t]
\caption{Full-supervised node classification accuracy (\%) $\pm$ 95\% confidence interval.}\label{tab:full_nc}
\centering
\resizebox{0.98\textwidth}{!}{
\begin{tabular}{lcccccccc}
\toprule
\textbf{Method} & \textbf{Cham.} & \textbf{Squi.} & \textbf{Texas} & \textbf{Corn.} & \textbf{Actor} & \textbf{Cora} & \textbf{Cite.} & \textbf{Pubm.} \\
\midrule
MLP         & 46.59{$\pm$1.84}                             & 31.01{$\pm$1.18}                             & 86.81{$\pm$2.24}                             & 84.15{$\pm$3.05}                             & 40.18{$\pm$0.55}                             & 76.89{$\pm$0.97}                            & 76.52{$\pm$0.89}                             & 86.14{$\pm$0.25}                             \\
GCN         & 60.81{$\pm$2.95}                             & 45.87{$\pm$0.88}                             & 76.97{$\pm$3.97}                             & 65.78{$\pm$4.16}                             & 33.26{$\pm$1.15}                             & 87.18{$\pm$1.12}                            & 79.85{$\pm$0.78}                             & 86.79{$\pm$0.31}                             \\
APPNP       & 52.15{$\pm$1.79}                             & 35.71{$\pm$0.78}                             & 90.64{$\pm$1.70}                             & 91.52{$\pm$1.81}                             & 39.76{$\pm$0.49}                             & 88.16{$\pm$0.74}                            & 80.47{$\pm$0.73}                             & 88.13{$\pm$0.33}                             \\
\midrule
ARMA	&60.21{$\pm$1.00}	&36.27{$\pm$0.62}	&83.97{$\pm$3.77}	&85.62{$\pm$2.13}	&37.67{$\pm$0.54}	&87.13{$\pm$0.80}	&80.04{$\pm$0.55}	&86.93{$\pm$0.24}\\
GPR-GNN     & 67.49{$\pm$1.38}                             & 50.43{$\pm$1.89}                             & 92.91{$\pm$1.32}                             & 91.57{$\pm$1.96}                             & 39.91{$\pm$0.62}                             & 88.54{$\pm$0.67}                            & 80.13{$\pm$0.84}                             & 88.46{$\pm$0.31}                             \\
BernNet     & 68.53{$\pm$1.68}                             & 51.39{$\pm$0.92}                             & 92.62{$\pm$1.37}                             & 92.13{$\pm$1.64}                             & 41.71{$\pm$1.12}                             & 88.51{$\pm$0.92}                            & 80.08{$\pm$0.75}                             & 88.51{$\pm$0.39}                             \\
ChebNetII   & 71.37{$\pm$1.01}                             & 57.72{$\pm$0.59}                             & 93.28{$\pm$1.47}                             & 92.30{$\pm$1.48}                             & 41.75{$\pm$1.07}                             & 88.71{$\pm$0.93}                            & 80.53{$\pm$0.79}                             & 88.93{$\pm$0.29}                             \\
JacobiConv                              & 74.20{$\pm$1.03}                             & 57.38{$\pm$1.25}                             & \underline{93.44{$\pm$2.13}}                             & \underline{92.95{$\pm$2.46}}                             & 41.17{$\pm$0.64}                             & 88.98{$\pm$0.46}                            & 80.78{$\pm$0.79}                             & 89.62{$\pm$0.41}                             \\
Specformer     &\underline{75.06{$\pm$1.10}}	&\textbf{65.05{$\pm$0.96}}	&90.33{$\pm$3.12}	&90.00{$\pm$2.79}	&42.55{$\pm$0.67}	&88.85{$\pm$0.46}	&80.68{$\pm$0.90}	&91.25{$\pm$0.31} \\
LON-GNN     & 73.00{$\pm$2.20}  & 60.61{$\pm$1.69}  & 87.54{$\pm$3.45}  & 84.47{$\pm$3.45} & 39.10{$\pm$1.59}  & \underline{89.44{$\pm$1.12}}  & \underline{81.41{$\pm$1.15}}  & 90.98{$\pm$0.64}  \\
OptBasisGNN & 74.26{$\pm$0.74} & 63.62{$\pm$0.76} & 91.15{$\pm$1.97}  & 89.84{$\pm$2.46} & 42.39{$\pm$0.52} & 87.96{$\pm$0.71}     & 80.58{$\pm$0.82}  & 90.30{$\pm$0.19} \\
FLODE	&74.38$\pm$0.92      &63.09$\pm$1.04      &89.34$\pm$1.32      &89.02$\pm$2.95      &\underline{42.75$\pm$1.10}      &88.28$\pm$1.02      &80.66$\pm$1.16      &90.59$\pm$0.55      \\
\midrule
GCNII       & 63.44{$\pm$0.85}                             & 41.96{$\pm$1.02}                             & 80.46{$\pm$5.91}                             & 84.26{$\pm$2.13}                             & 36.89{$\pm$0.95}                             & 88.46{$\pm$0.82}                            & 79.97{$\pm$0.65}                             & 89.94{$\pm$0.31}                             \\
PDE-GCN     & 66.01{$\pm$1.56}                             & 48.73{$\pm$1.06}                             & 93.24{$\pm$2.03}                             & 89.73{$\pm$1.35}                             & 39.76{$\pm$0.74}                             & 88.62{$\pm$1.03}                            & 79.98{$\pm$0.97}                             & 89.92{$\pm$0.38}                             \\
NodeFormer                              & 53.02{$\pm$1.58}                            & 34.25{$\pm$1.96}                            & 87.71{$\pm$2.13}                            & 90.00{$\pm$3.45}                            & 41.74{$\pm$0.61}                            & 86.93{$\pm$1.22}                           & 79.58{$\pm$0.85}                            & \underline{91.27{$\pm$0.39}}                            \\
GloGNN++    & 72.36{$\pm$0.85}                            & 60.60{$\pm$1.04}                            & 91.48{$\pm$1.48}                            & 89.84{$\pm$3.62}                            & 41.87{$\pm$1.02}                            & 87.21{$\pm$0.59}                           & 79.89{$\pm$0.61}                            & 86.89{$\pm$0.33}                            \\
LRGNN	& 75.01$\pm$0.64	& 63.32$\pm$1.27	& 91.80$\pm$2.46	& 90.33$\pm$1.15	& 41.16$\pm$0.91	& 88.77$\pm$0.94	& 79.85$\pm$0.83	& 90.73$\pm$0.47 \\
\midrule
GNN-LF      & 53.74{$\pm$1.29}  & 36.15{$\pm$0.86}  & 76.07{$\pm$2.62} & 78.36{$\pm$2.46} & 38.39{$\pm$0.81}  & 88.51{$\pm$0.89}  & 79.84{$\pm$0.56}  & 89.86{$\pm$0.23}  \\
GNN-HF      & 55.97{$\pm$1.05}  & 35.29{$\pm$0.72}  & 81.15{$\pm$2.62}  & 85.41{$\pm$3.12} & 38.96{$\pm$0.77}  & 88.28{$\pm$0.64}  & 80.04{$\pm$0.93}  & 90.35{$\pm$0.30}  \\
ADA-UGNN    & 61.09{$\pm$1.51}  & 42.02{$\pm$1.26}  & 84.92{$\pm$3.12}  & 83.61{$\pm$3.44} & 41.10{$\pm$0.62}  & 88.74{$\pm$0.85}  & 79.81{$\pm$1.11}  & 90.61{$\pm$0.44}  \\
FE-GNN      & 73.00{$\pm$1.31}  & 63.28{$\pm$0.81}  & 88.03{$\pm$1.80}  & 86.07{$\pm$3.12} & 41.74{$\pm$0.67}  & 89.21{$\pm$0.71} & 80.26{$\pm$1.06}  & 90.80{$\pm$0.30}  \\
\midrule
SAF	&\textbf{75.30{$\pm$0.96}}	&63.63{$\pm$0.81}	&94.10{$\pm$1.48}	&92.95{$\pm$1.97}	&42.93{$\pm$0.79}	&89.80{$\pm$0.69}	&80.61{$\pm$0.81}	&91.49{$\pm$0.29}\\
SAF-$\epsilon$    & 74.84{$\pm$0.99}                           &\underline{64.00{$\pm$0.83}}                             & \textbf{94.75{$\pm$1.64}}                            & \textbf{93.28{$\pm$1.80}}                            & \textbf{42.98{$\pm$0.61}}                            & \textbf{89.87{$\pm$0.51}}                           & \textbf{81.45{$\pm$0.59}}                            & \textbf{91.52{$\pm$0.30}}                            \\
Improv.\footnotemark[3]    & 6.77\%     & 12.61\%    & 2.13\%     & 1.15\%     & 1.27\%     & 1.36\%    & 1.37\%     & 3.01\%                                            \\
\bottomrule       
\end{tabular}}
\end{table*}

\subsection{Semi-supervised Node Classification.}
In this task, we follow the experimental protocol established by~\cite{he2022chebnetii} and compare our models with MLP, two basic GNNs, eight popular polynomial spectral GNNs, and five unified GNNs. For data splitting on homophilic graphs (Cora, Citeseer, and Pubmed), we apply the standard division~\cite{yang2016revisiting} with 20 nodes per class for training, 500 nodes for validation, and 1,000 nodes for testing. On the other five heterophilic graphs, we leverage the sparse splitting~\cite{chien2021adaptive} with 2.5\%/2.5\%/95\% samples respectively for training/validation/testing. The results are reported in Table~\ref{tab:semi_nc}, where the best results are bold and the underlined letters denote the second highest accuracy. We first observe that both SAF and SAF-$\epsilon$ substantially boosts its base model, BernNet, with gains reaching a notable 15.37\%. This impressive enhancement is credited to their capacity to effectively exploit the task-beneficial information, which is implicitly encoded by spectral filtering in the spatial domain. This ability is particularly advantageous in contexts with limited supervision, where it allows effective leveraging of extra prior knowledge during training. 
Generally, our models outperform competitors on all datasets except for Chameleon and Squirrel, where SAF maintains a second-place rank with considerable improvements on BernNet by 14.56\% and 9.68\%. In these cases, ChebNetII initially surpasses our model, yet, with more training samples, our SAF manages to beats it by margins of 3.93\% and 6.28\% (see Table~\ref{tab:full_nc}).
Moreover, SAF-$\epsilon$ averagely delivers better results than SAF by using thresholding sparsity to reduce non-local noise. However, this enhancement also incurs higher computational costs, as illustrated in both Section~\ref{sec:complexity} and Section~\ref{apdix:run_time}.

\subsection{Full-supervised Node Classification.}
To bolster our evaluation, we expand the previously compared baselines to include five cutting-edge spatial GNN models: GCNII \& PDE-GCN capturing long-range dependency, and NodeFormer, GloGNN++ \& LRGNN, which go further by not only capturing long-range dependencies but also effectively addressing graph heterophily.
For all datasets, we randomly divide them into 60\%/20\%/20\% for training/validation/testing by following~\cite{he2021bernnet,he2022chebnetii}. Table~\ref{tab:full_nc} summarizes the mean classification accuracies. Our methods demonstrate superior performance across most datasets, with an exception on Squirrel where they achieve comparable results to Specformer. 
This notable performance is primarily attributed to our SAF's effective non-local aggregation, utilizing signed edge weights to model global label relationships. This enables our methods to outperform GNNs that are specifically tailored for long-range dependency and/or graph heterophily.

\footnotetext[3]{Improv. indicates the relative improvement of our SAF over its base model, BernNet~\cite{he2021bernnet}. For alternative implementation using ChebNetII~\cite{he2022chebnetii} as backbone, please refer to Section~\ref{apdix:chebNetAsBase}.}

\begin{figure}[t]
\centering
\subfloat[Squirrel]{\includegraphics[width=0.24\textwidth]{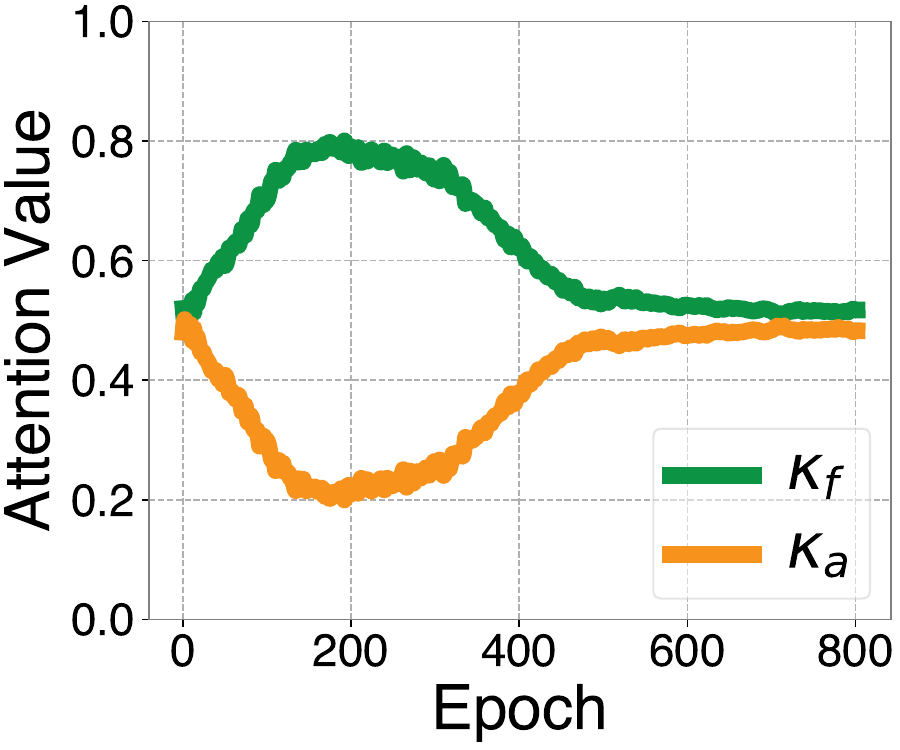}}
\subfloat[Cora]{\includegraphics[width=0.24\textwidth]{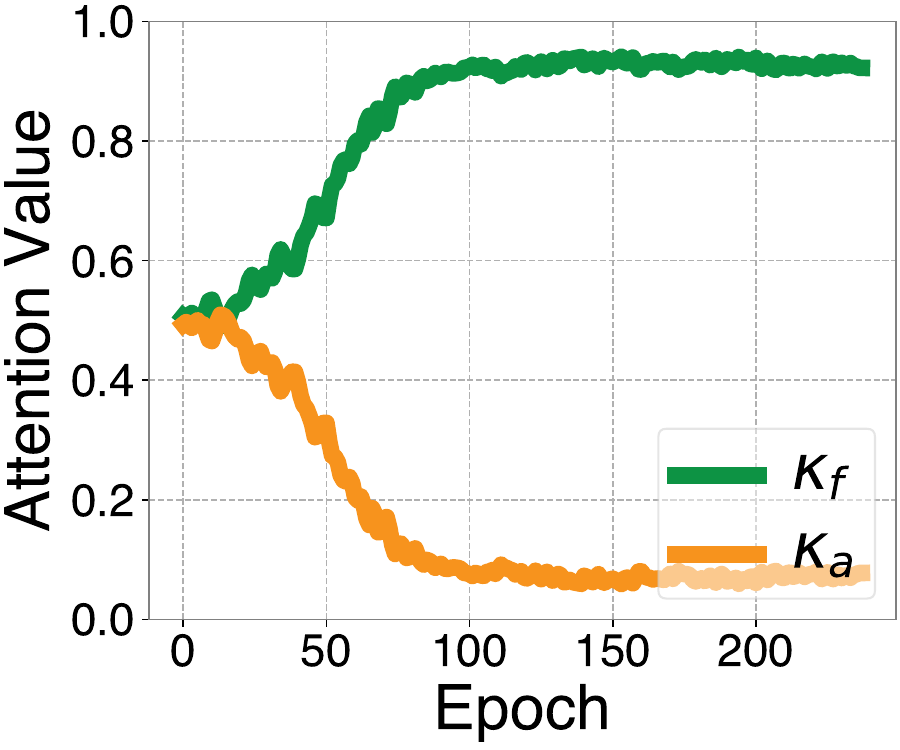}}
\caption{Attention changing trends w.r.t. training epochs.}
\label{fig:additional_results} 
\end{figure}

\subsection{Analysis of Attention Trends}
We analyze the changing trends of the pair-wise attention weights during training SAF on Squirrel and Cora datasets.
From Fig.~\ref{fig:additional_results}(b-c), the average weights for filtering and aggregation start similarly but diverge throughout training, showing different trends in heterophilic and homophilic graphs. On the heterophilic graph Squirrel, both weights converge to similar values, demonstrating their mutual importance in modeling complex connectivity. 
Conversely, $\boldsymbol{\kappa}_f$ becomes dominant on the homophilic graph Cora due to the sufficiency of node proximity information for label prediction, thereby diminishing the relevance of $\boldsymbol{\kappa}_a$ and non-local aggregation.

\begin{figure}[t]
\centering
\includegraphics[width=0.48\textwidth]{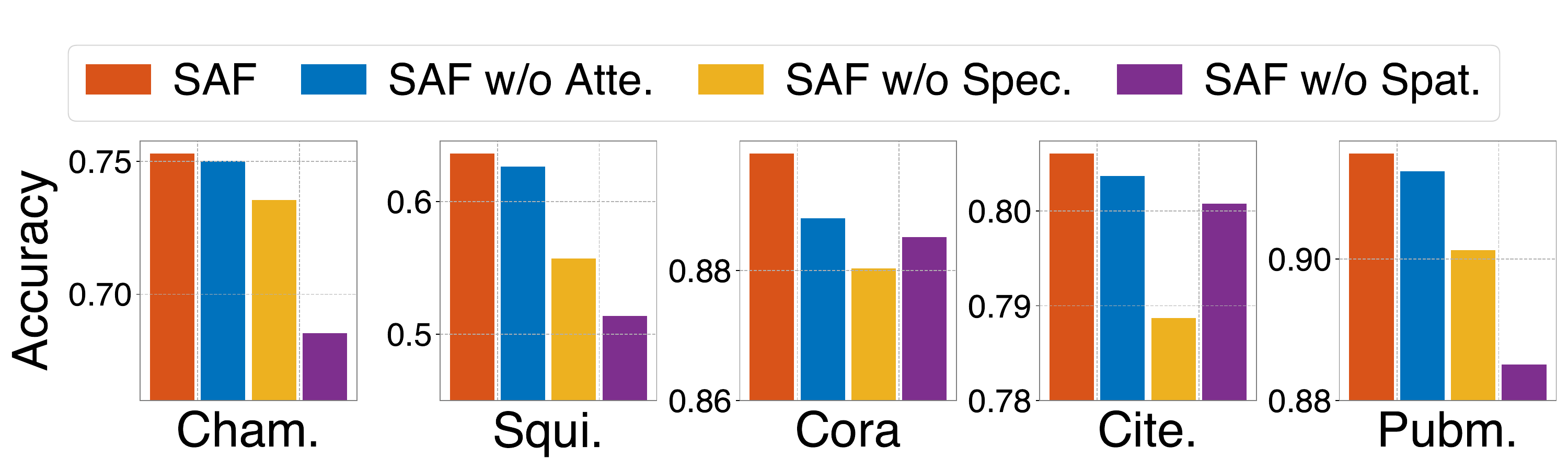}
\caption{Ablation study of SAF framework on five datasets.}
\label{fig:abla}
\end{figure}

\begin{figure}[t]
\centering
\subfloat[Chameleon]{\includegraphics[width=0.155\textwidth]{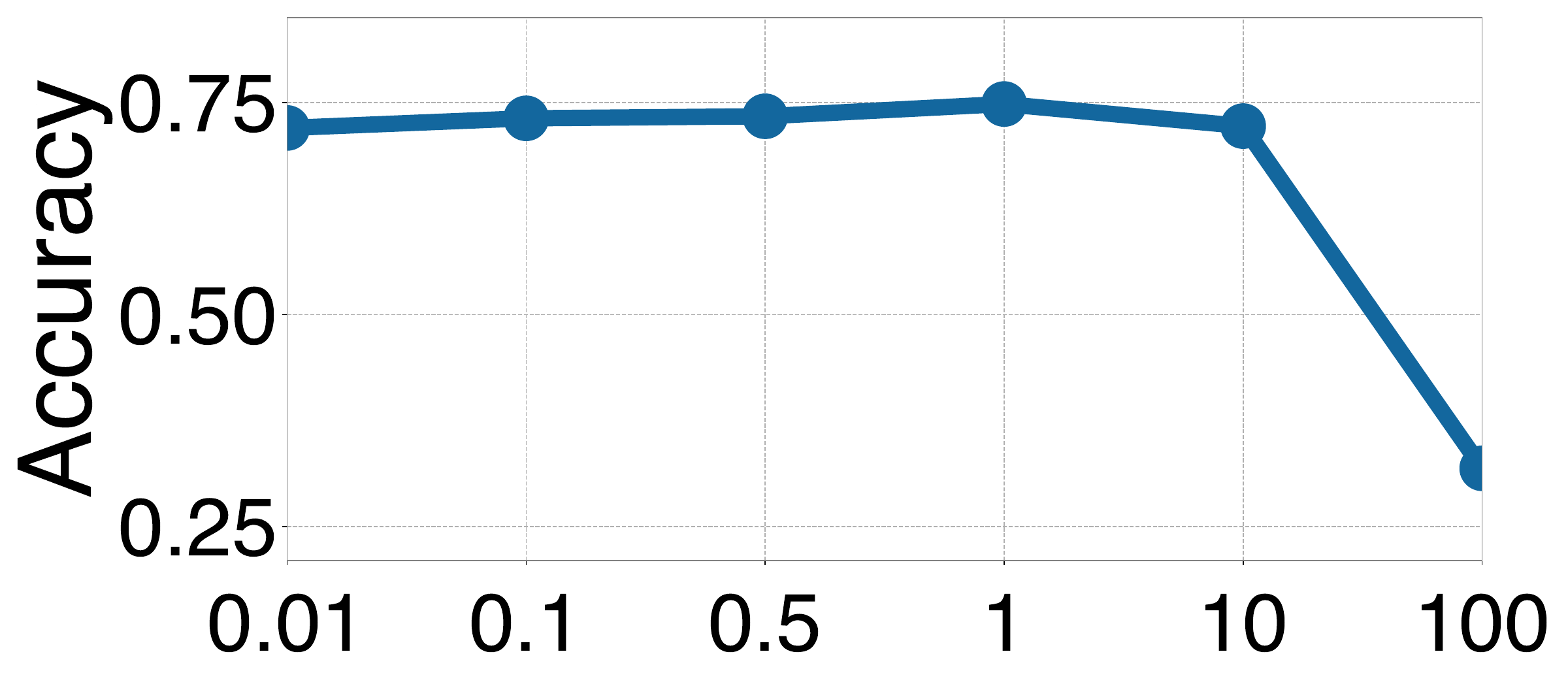}}
\subfloat[Texas]{\includegraphics[width=0.155\textwidth]{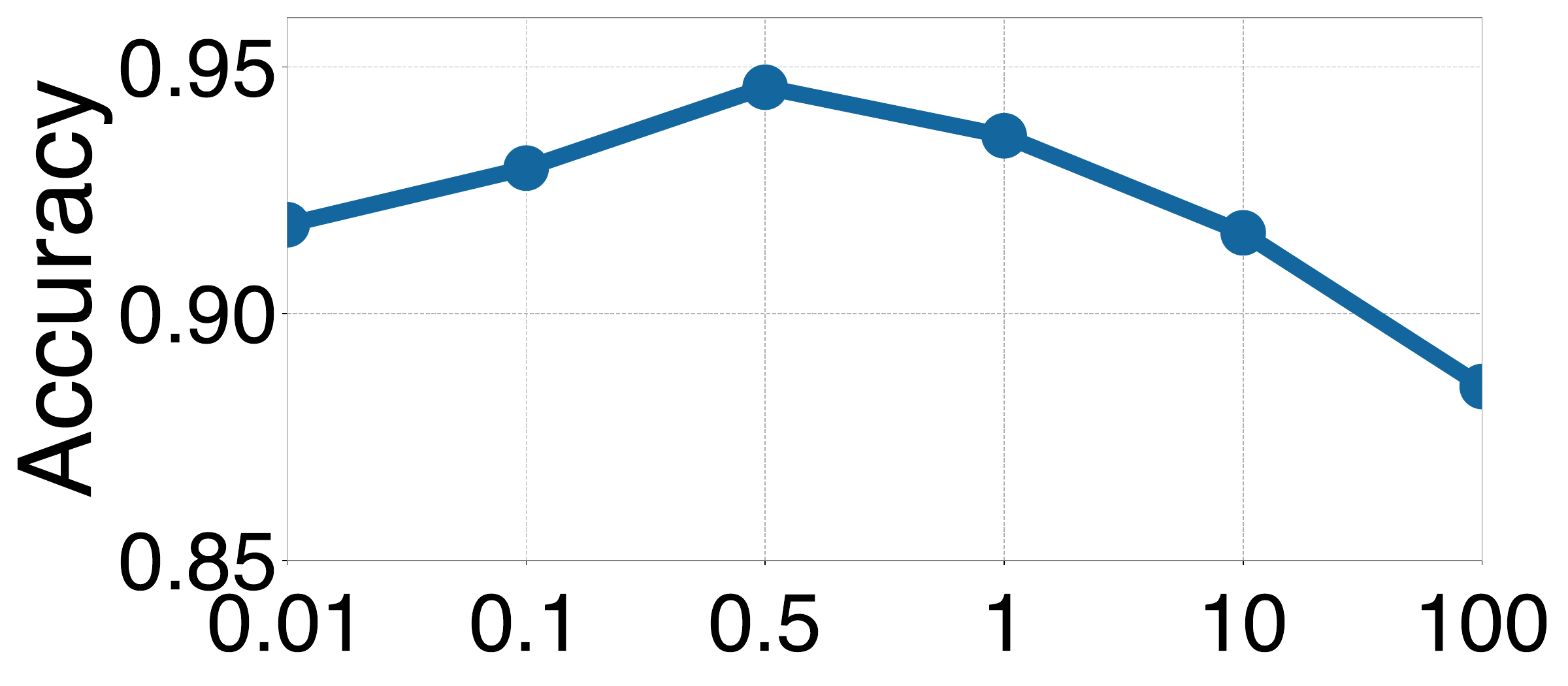}}
\subfloat[Cora]{\includegraphics[width=0.155\textwidth]{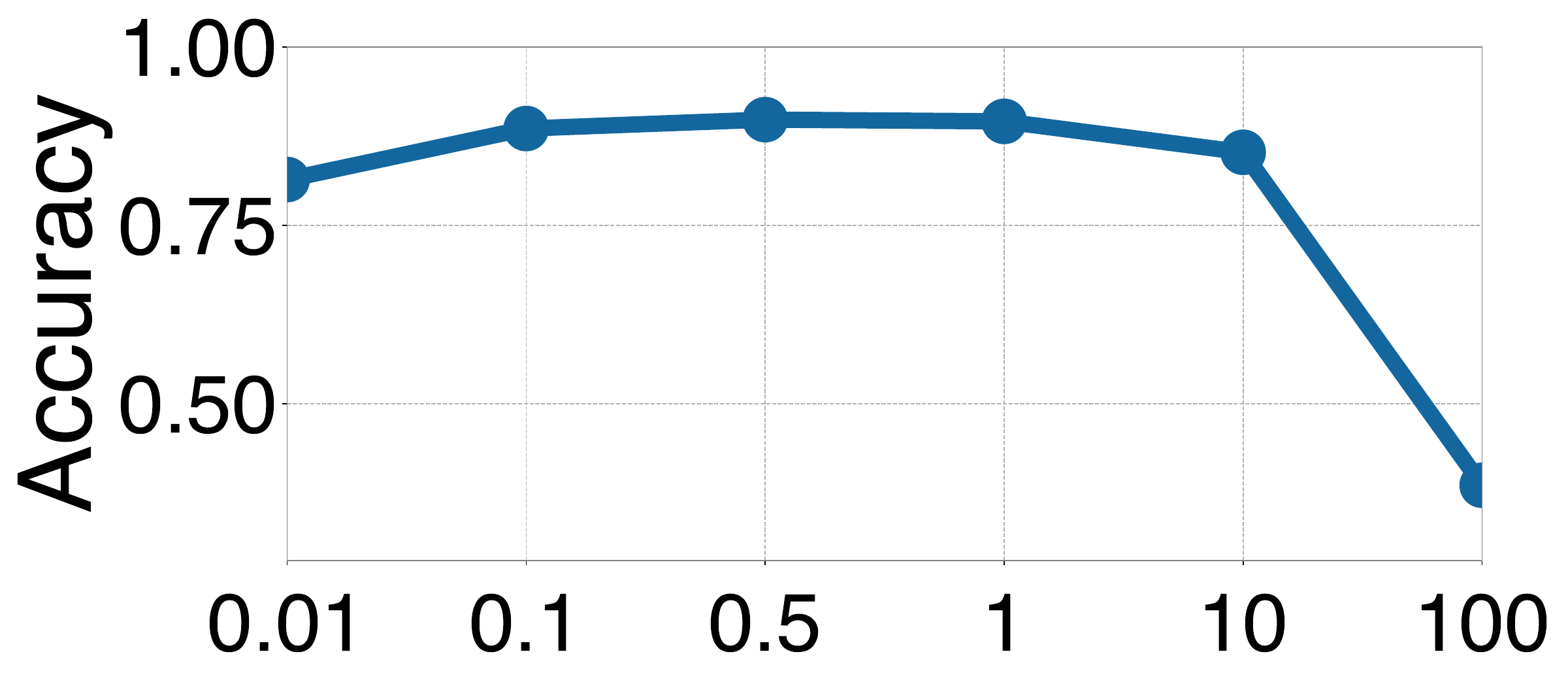}}
\hfill
\subfloat[Chameleon]{\includegraphics[width=0.155\textwidth]{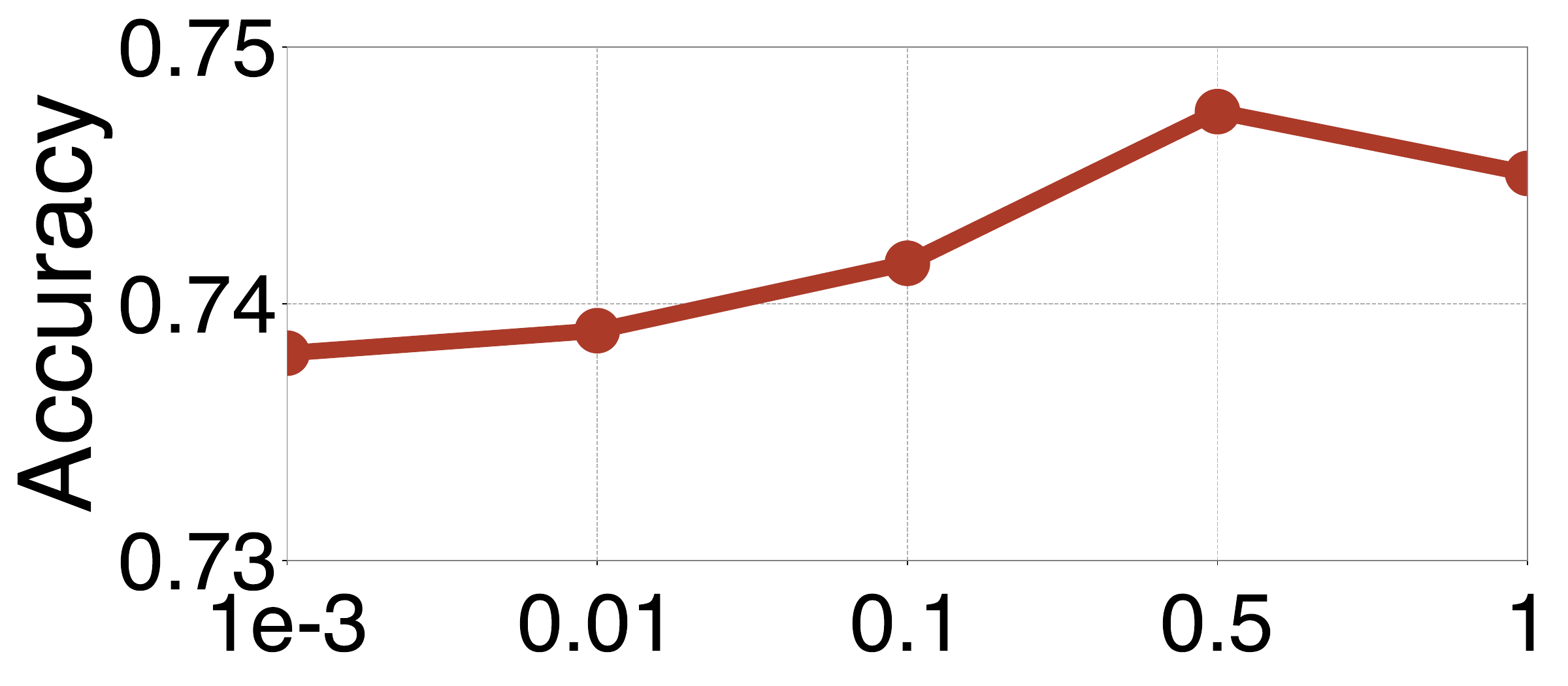}}
\subfloat[Texas]{\includegraphics[width=0.155\textwidth]{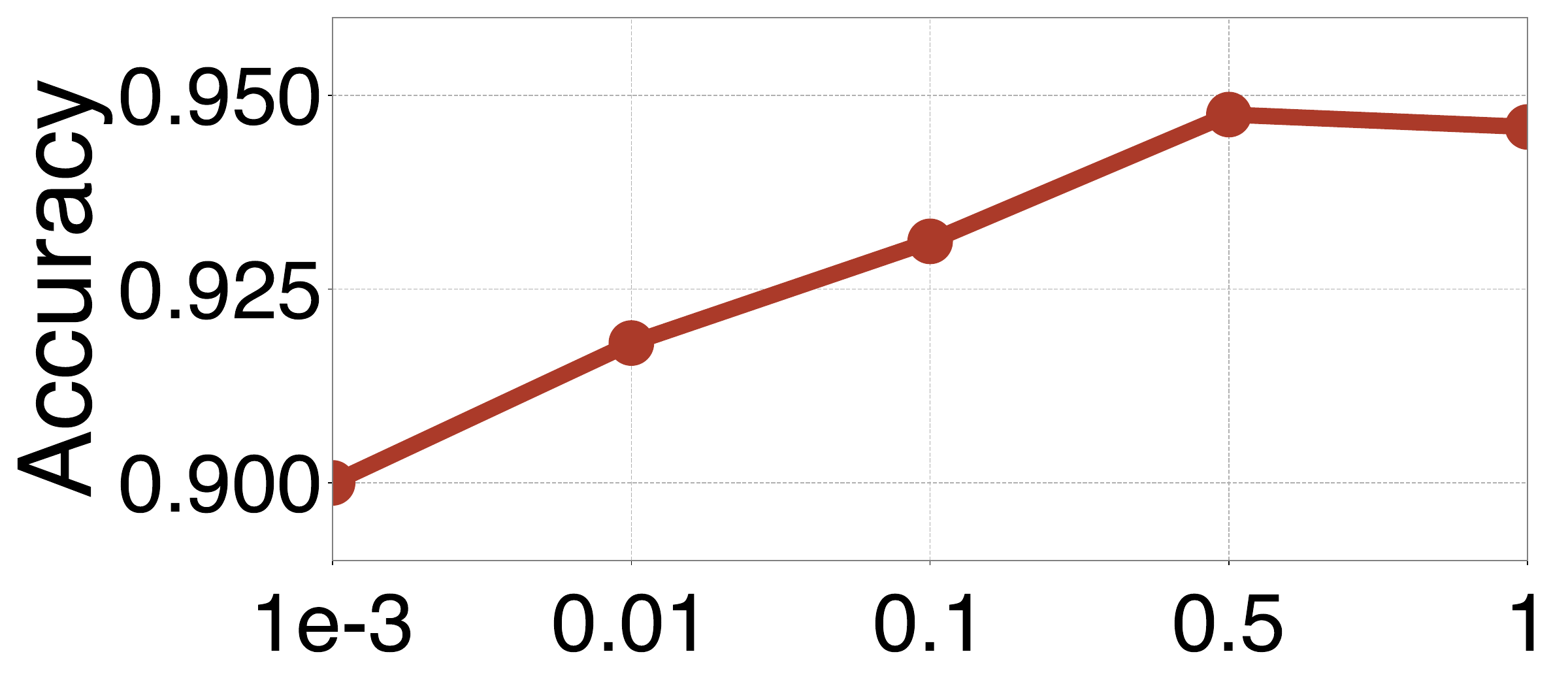}}
\subfloat[Cora]{\includegraphics[width=0.155\textwidth]{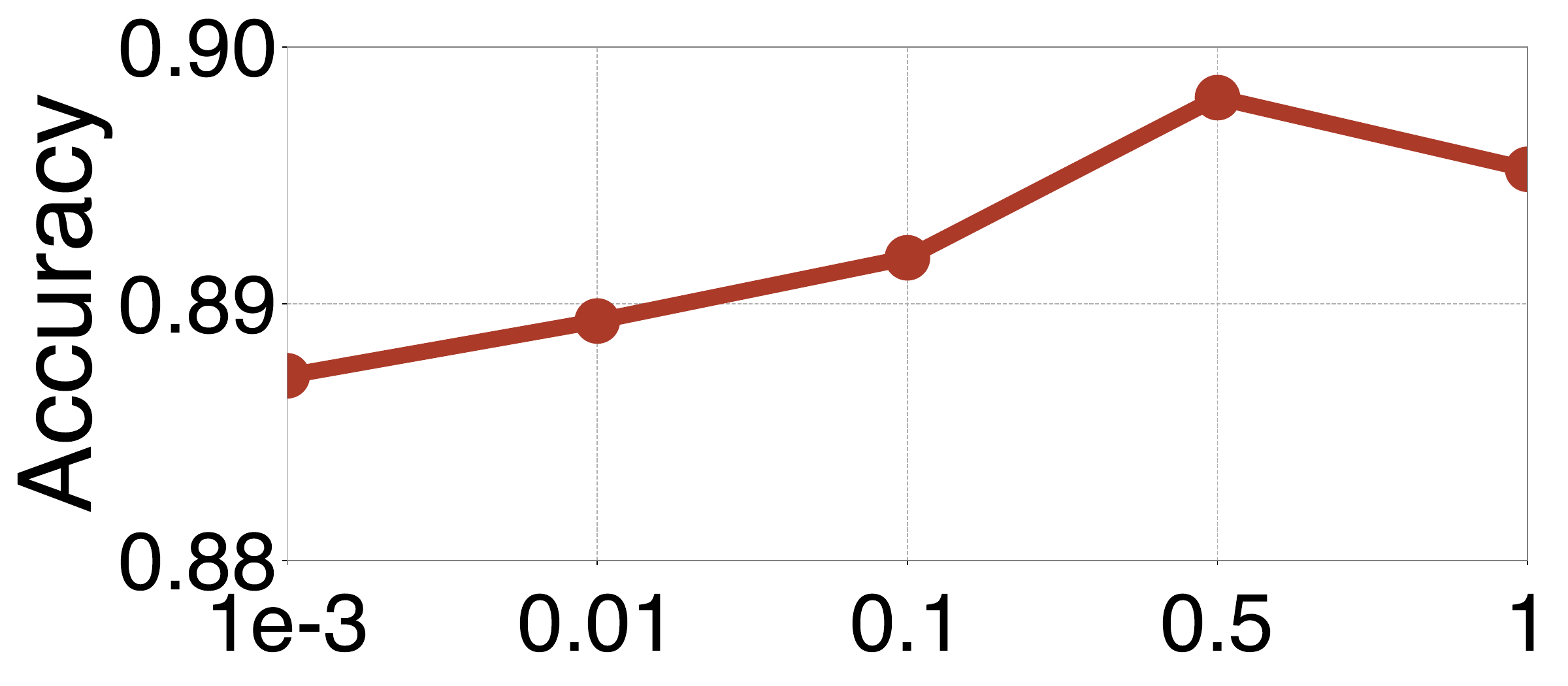}}
\hfill
\subfloat[Chameleon]{\includegraphics[width=0.155\textwidth]{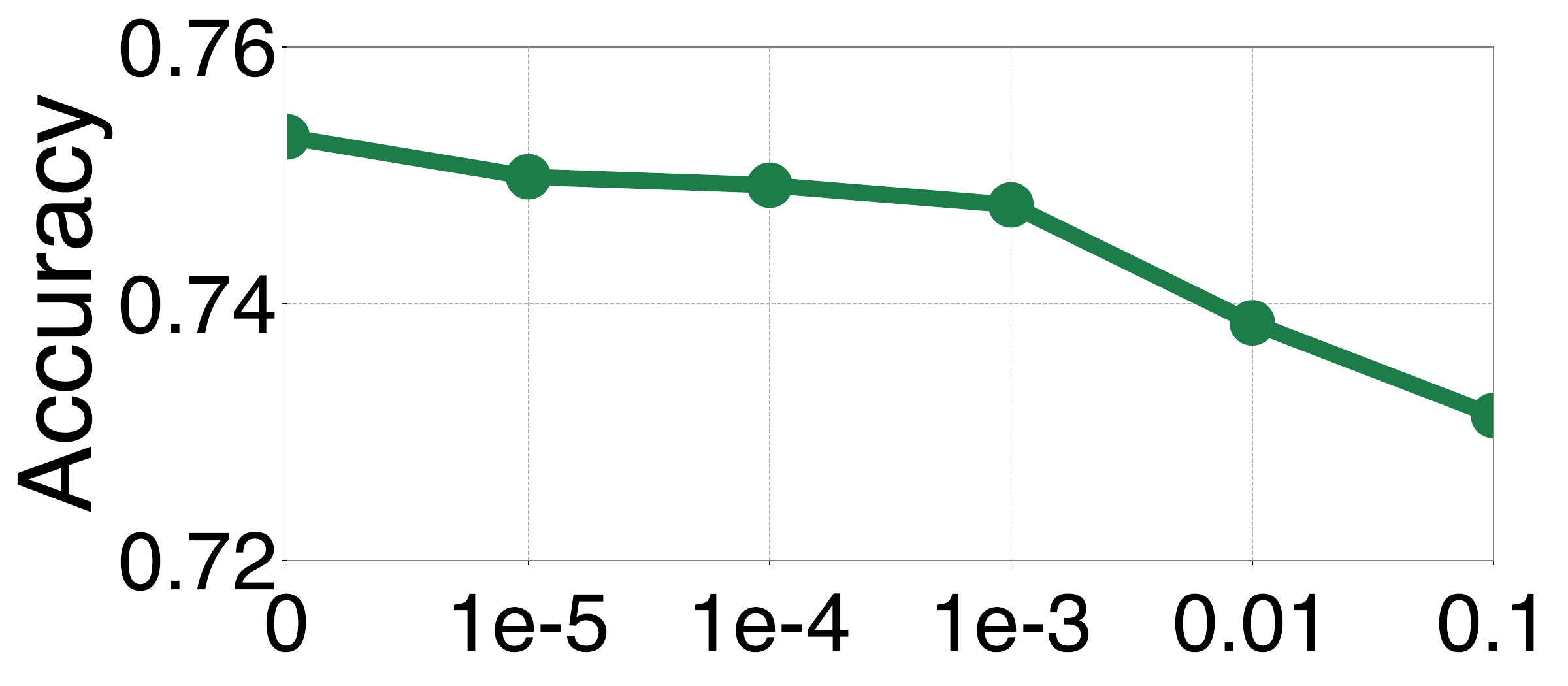}}
\subfloat[Texas]{\includegraphics[width=0.155\textwidth]{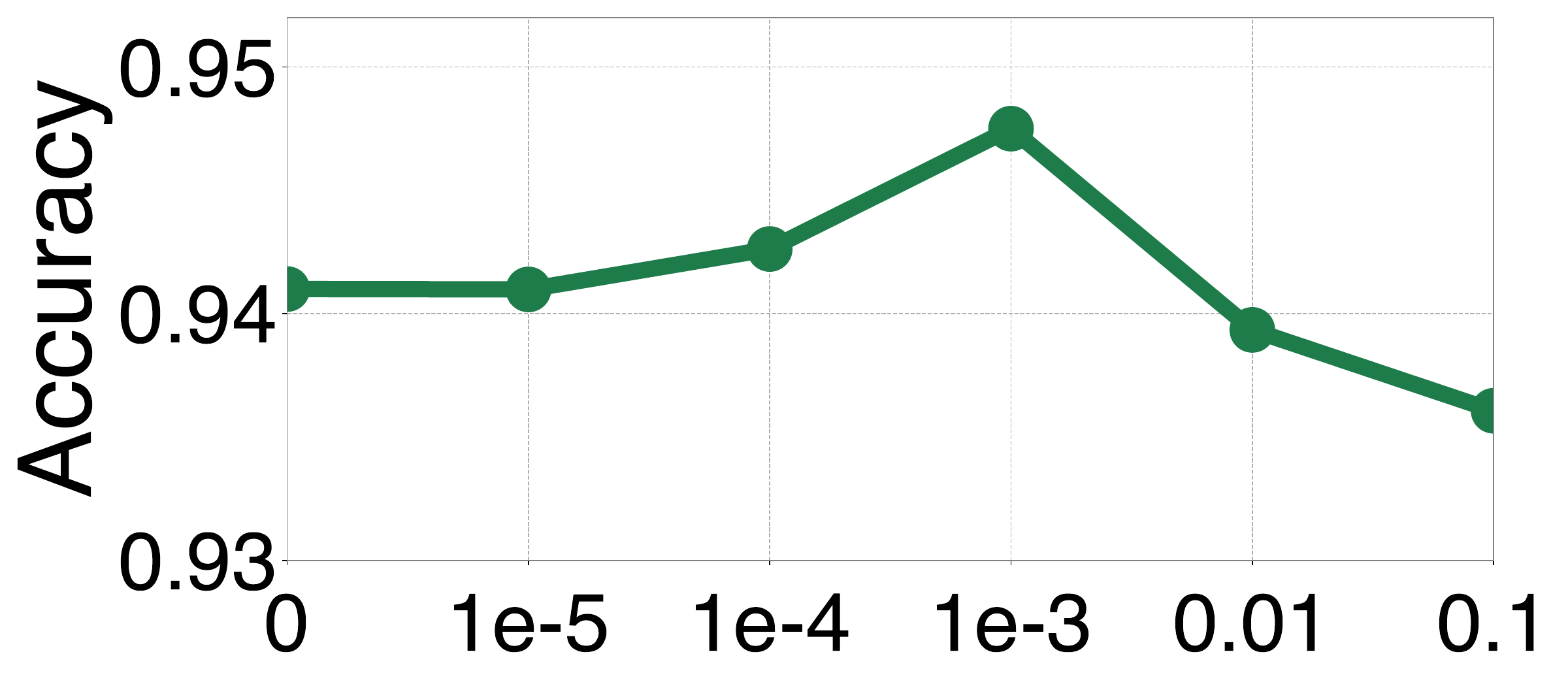}}
\subfloat[Cora]{\includegraphics[width=0.155\textwidth]{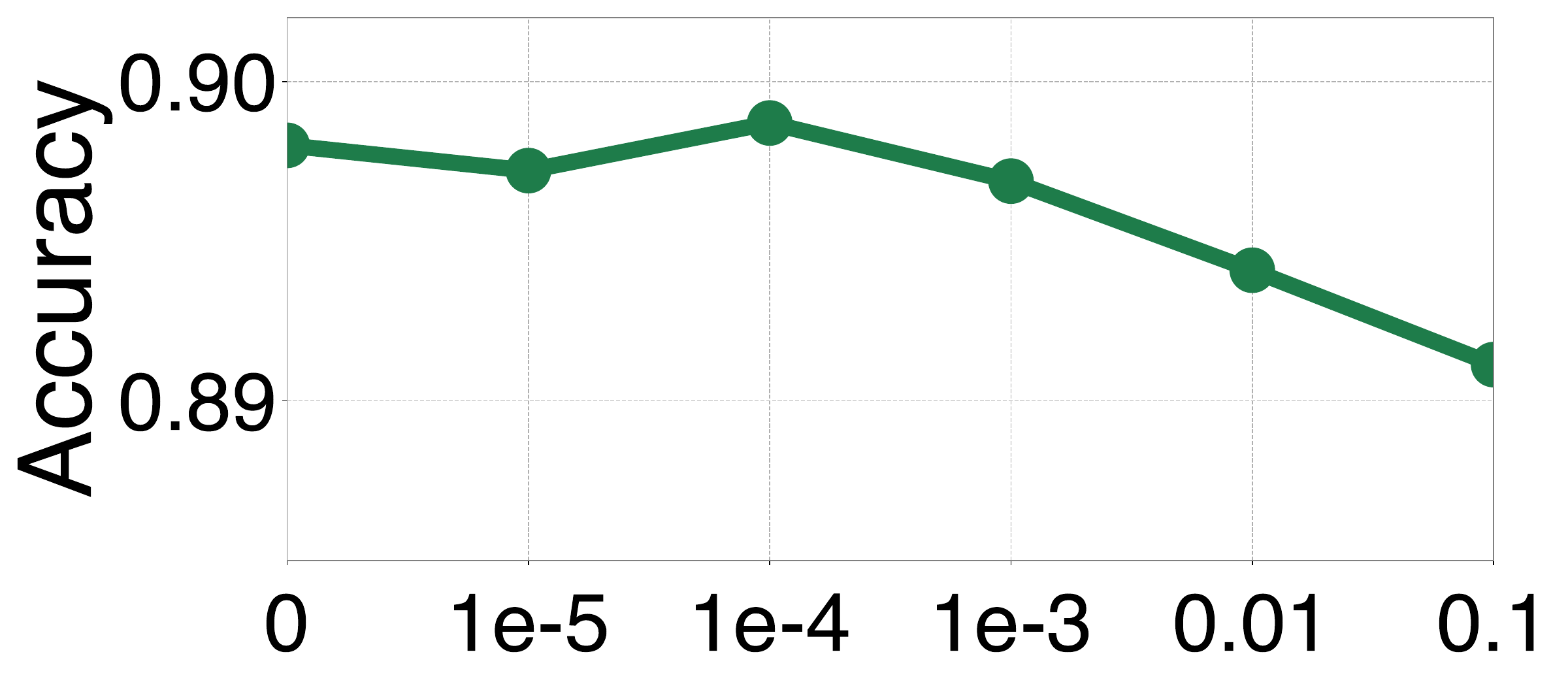}}
\hfill
\subfloat[Chameleon]{\includegraphics[width=0.155\textwidth]{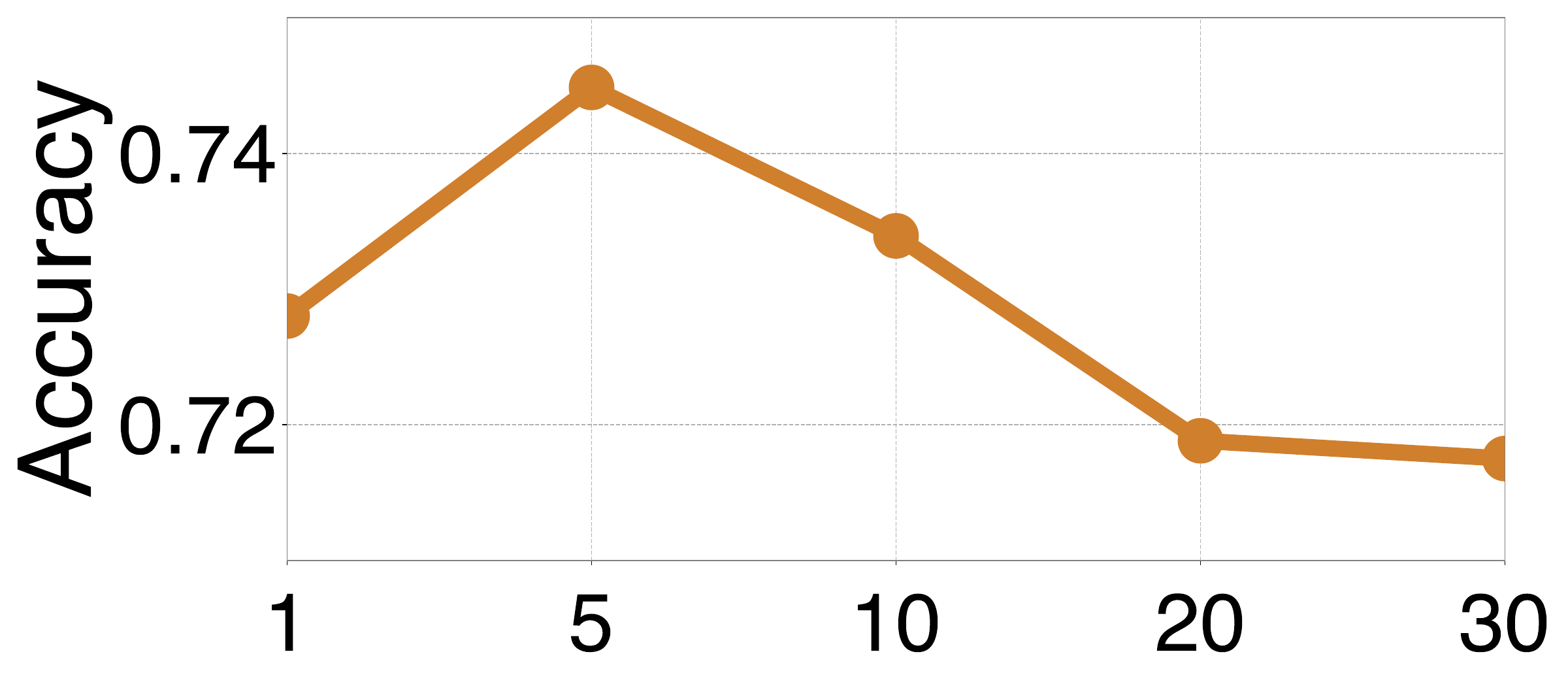}}
\subfloat[Texas]{\includegraphics[width=0.155\textwidth]{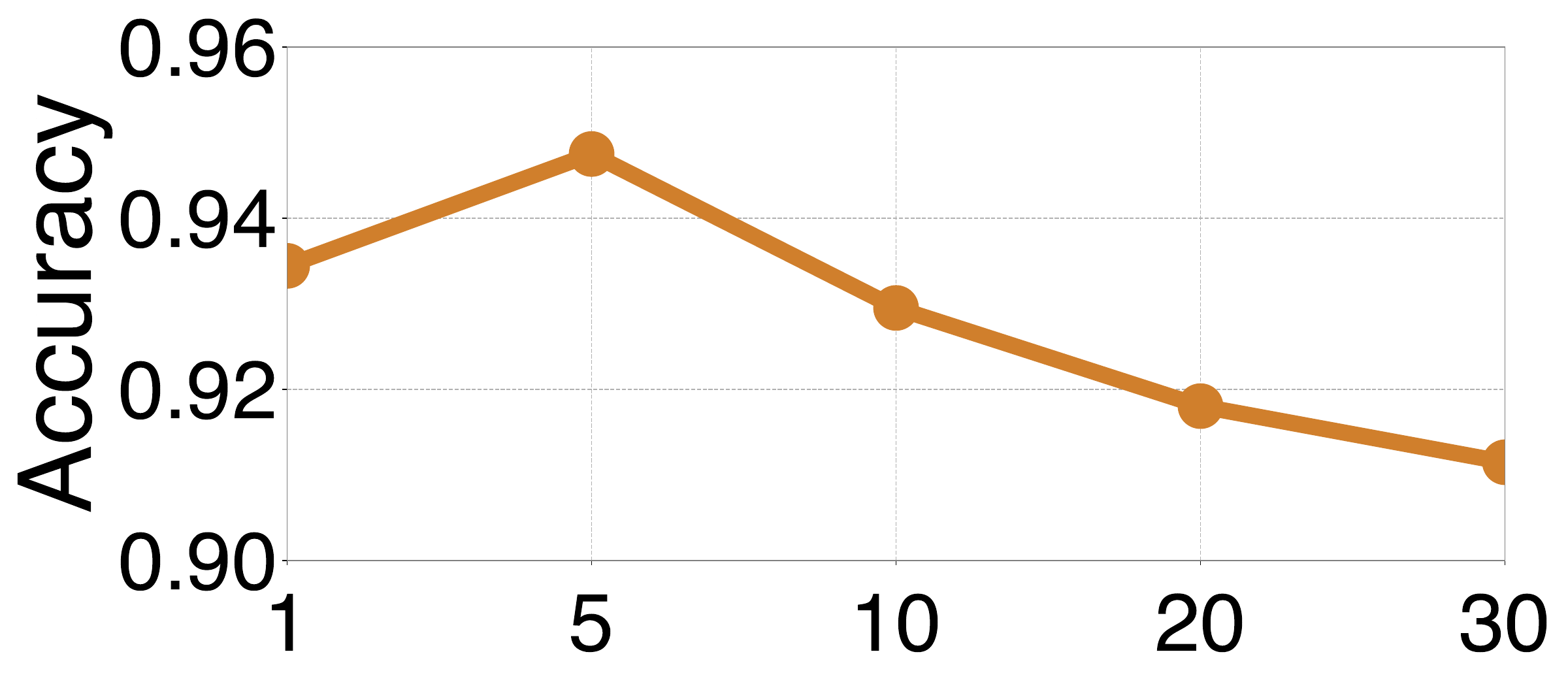}}
\subfloat[Cora]{\includegraphics[width=0.155\textwidth]{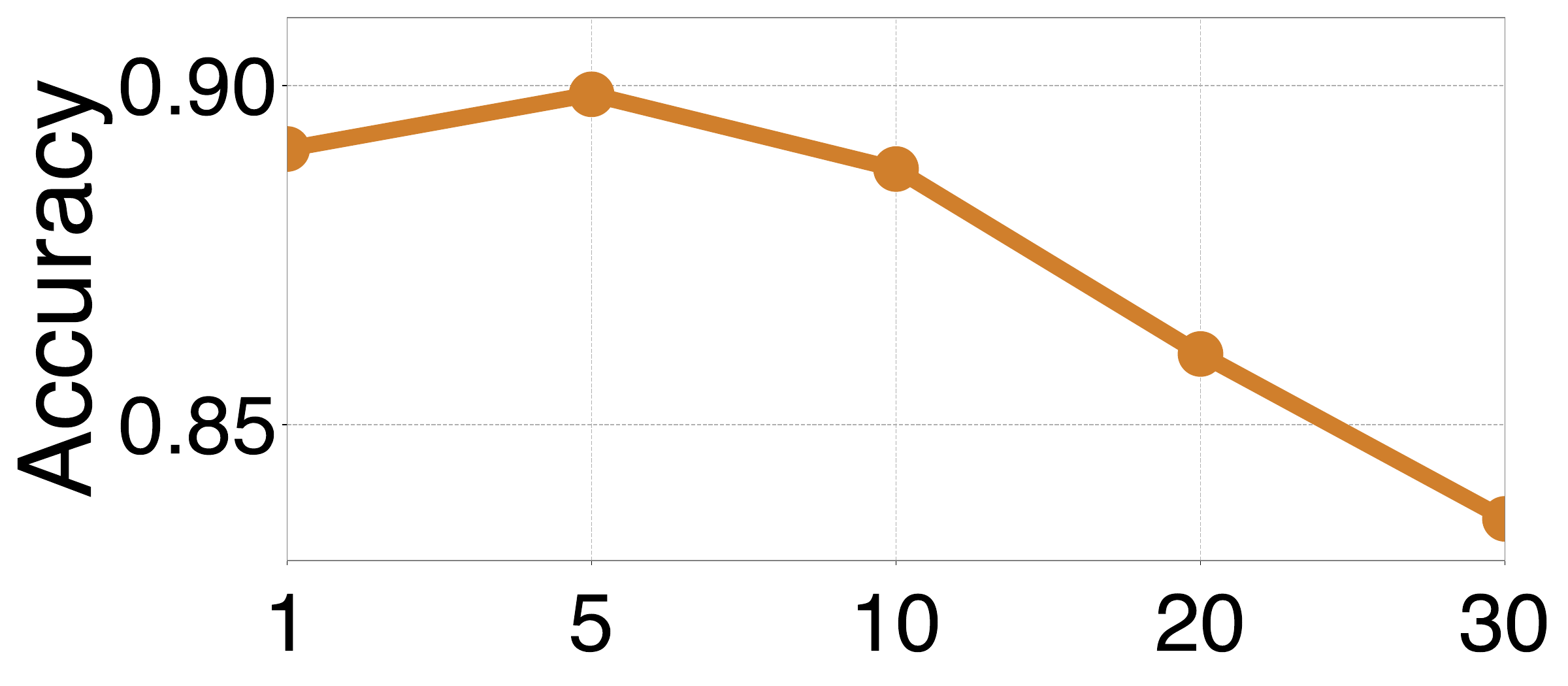}}
\caption{Sensitivity analysis of hyper-parameters: $\tau$, $\eta$, $\epsilon$, and $L$ from top to bottom rows.}\label{fig:sensi}
\end{figure}

\subsection{Ablation Study}
This section aims to validate our designs by comparing SAF with its three ablated variants -- SAF w/o Atte., SAF w/o Spec., and SAF w/o Spat. -- in full-supervised node classification. Specifically, Atte., Spec., and Spat. respectively refers to: attention mechanism in ``Node-wise Prediction Amalgamation", ``Non-negative Spectral Filtering", and ``Non-local Spatial Aggregation". For SAF w/o Atte., we remove the attention mechanism and equally blend predictions from different domains. SAF w/o Spec. abandons the spectral filtering phase, practically setting $\boldsymbol{\kappa_f}=\mathbf{0}, \boldsymbol{\kappa_a}=\mathbf{1}$.
As the SAF w/o Spat. configuration is equivalent to BernNet model, the corresponding results are posted directly. From Fig.~\ref{fig:abla}, we can draw several insights: \textbf{1)} The impact of Atte. module on our SAF varies by datasets, e.g., on Chameleon and Squirrel, showing a slight performance reduction upon its removal. This observation aligns with our observation that their optimal attention values are close to an even split, as suggested in Figures~\ref{fig:additional_results}(a).
Conversely, Cora dataset exhibits a notable drop, due to its optimal attention weights being far from even, as depicted in Fig.~\ref{fig:additional_results}(b); \textbf{2)} Spectral filtering (Spec. module) remains vital for discriminative node representation learning. Specifically, the quality of the adapted new graph fundamentally hinges on the graph spectral filters' training, as underscored by their theoretical interaction in Eq.~(\ref{eq:theo_interplay}). Practically, the absence of spectral filtering markedly reduces model accuracy, confirming its importance in SAF; \textbf{3)} This visualization not only reaffirms the pivotal role of the non-local aggregation (Spat. module),
but also underscores its position as the most crucial component in advancing spectral GNNs within the SAF framework.

\begin{table}[t]
\caption{Evaluations on new heterophilic graph datasets.}\label{tab:newheter_nc}
\centering
\setlength\tabcolsep{2pt}
\resizebox{0.48\textwidth}{!}{
\begin{tabular}{lccccc}
\toprule
\textbf{Method} & \textbf{Mine.} & \textbf{Tolo.} & \textbf{Amaz.} & \textbf{Roma.} & \textbf{Penn94} \\
\midrule
MLP         & 50.61{$\pm$0.87}      & 74.58{$\pm$0.69}   & 45.50{$\pm$0.38}         & 66.11{$\pm$0.33}       & 74.58$\pm$0.37\\
GCN         & 72.25{$\pm$0.60}      & 76.56{$\pm$0.85}   & 48.06{$\pm$0.39}         & 53.49{$\pm$0.33}       & 82.47$\pm$0.27\\
APPNP       & 68.48{$\pm$1.20}      & 74.13{$\pm$0.62}   & 48.12{$\pm$0.37}         & 72.99{$\pm$0.46}       & 75.29$\pm$0.27\\
GPR-GNN     & 89.76{$\pm$0.53}      & 75.82{$\pm$0.50}   & 49.06{$\pm$0.25}         & 73.19{$\pm$0.24}       & 81.38$\pm$0.16\\
ChebNetII   & 83.62{$\pm$1.51}      & 78.95{$\pm$0.49}                              & 49.76{$\pm$0.36}         & 74.52{$\pm$0.54}      & 83.12$\pm$0.22\\
JacobiConv  & 89.88{$\pm$0.33}      & 77.24{$\pm$0.39}   & 43.89{$\pm$0.28}         & 74.30{$\pm$0.50}       & \underline{83.35$\pm$0.11}\\
NodeFormer  & \underline{89.89{$\pm$0.46}}      & \textbf{80.31{$\pm$0.75}}   & 43.67{$\pm$1.54}         & 73.59{$\pm$0.60}       & 69.66{$\pm$0.83}                                \\
GloGNN++    & 72.59{$\pm$1.54}      & 79.01{$\pm$0.84}   &\underline{50.03{$\pm$0.29}}                     & 66.10{$\pm$0.26}       & 73.15{$\pm$0.59}                                \\
FE-GNN	&84.68$\pm$0.36	&79.31$\pm$0.59	&49.46$\pm$0.29	&74.50$\pm$0.30	     &82.30$\pm$0.54\\
\midrule
BernNet     & 77.75{$\pm$0.61}      & 75.35{$\pm$0.63}   & 49.84{$\pm$0.52}         & \underline{74.56{$\pm$0.74}}      &82.47$\pm$0.21     \\
SAF	&\textbf{90.54{$\pm$0.30}}	&\underline{79.38$\pm$0.58}	&\textbf{50.49{$\pm$0.28}}	&\textbf{74.87$\pm$0.22}	&\textbf{83.86$\pm$0.26} \\
Improv.     & 12.79\%          & 4.03\%        & 0.65\%              & 0.31\%            & 1.39\%\\
\bottomrule
\end{tabular}}
\end{table}

\subsection{Parameter Study}
This section presents the sensitivity analysis of hyper-parameters including $\tau$, $\eta$, $\epsilon$, and $L$. Fig.~\ref{fig:sensi} visualizes how varying these parameters within a broad range influences learning performance, showcasing our model’s robust stability over diverse settings. 
Beyond empirical observation, we also provides deeper insights into parameter understanding and rationalizes the chosen ranges for parameter searching: 
\textbf{1)} The scaling parameter $\tau = \frac{\alpha}{1-\alpha}$, crucial in new graph construction in Eq.~(\ref{eq:theo_interplay}), stems from the trade-off parameter $\alpha \in (0, 1)$ within the graph optimization problem in Eq.~(\ref{eq:generalized_opt}). While theoretically we have $0 =\frac{0}{1-0} < \tau < \frac{1}{1-1} = \infty$, practical considerations for extracting structural information suggest a larger penalty on the trace objective term $\mathrm{tr}(\mathbf{Z}^T \gamma_{\theta}(\hat{\mathbf{L}}) \mathbf{Z})$, i.e., keeping $\alpha < 0.5$, thereby limiting $\tau < \frac{0.5}{1-0.5} = 1$. This rationale substantiates our selection of $\tau$ within the set $\{0.1, 0.2, ..., 1\}$ as stated in Section~\ref{apidx:param_searching}, aligning with the observed optimal performance 
in Figures~\ref{fig:sensi}(a)-(e).
When addressing graphs with noisy structure, we may adjust the upper limit of $\alpha$ to $t \in (0, 1)$, setting $\tau$'s maximum possible value to $\frac{t}{1-t}$.
For graph benchmarking evaluations in this work, where extracting structural information is important, we practically set $t=0.5$;
\textbf{2)} For the non-local aggregation layer number $L$, a noticeable decline in model performance is observed when $L$ exceeds 10. This is attributed to the non-local nature of our new graph, which facilitates efficient information exchange between nodes. Exceeding a certain number of layers may potentially lead to oversmoothing, where there is an overemphasis on global information, thus degrading model performance. However, choosing the number of layers within a reasonable range generally ensures consistent and impressive model performance, as verified in 
Figures~\ref{fig:sensi}(j)-(l).

\begin{table*}[t]
\caption{Full-supervised node classification accuracy (\%) while implementing SAF upon ChebNetII. 
}
\label{tab:chebNetAsBase}
\centering
\resizebox{0.98\textwidth}{!}{
\begin{tabular}{lcccccccc}
\toprule
\textbf{Method}    & \textbf{Cham.}       & \textbf{Squi.}       & \textbf{Texas}       & \textbf{Corn.}       & \textbf{Actor}       & \textbf{Cora}        & \textbf{Cite.}       & \textbf{Pubm.}       \\
\midrule
ChebNetII & 71.37$\pm$1.01  & 57.72$\pm$0.59  & 93.28$\pm$1.47  & 92.30$\pm$1.48  & 41.75$\pm$1.07  & 88.71$\pm$0.93  & 80.53$\pm$0.79  & 88.93$\pm$0.29  \\
SAF-Cheb  & 74.97$\pm$0.66 & 64.06$\pm$0.59 & 94.43$\pm$1.81 & 92.62$\pm$2.13 & 42.65$\pm$1.01 & 89.56$\pm$0.64 & 80.68$\pm$0.68 & 91.27$\pm$0.34 \\
SAF-Cheb-$\epsilon$
& 75.25$\pm$0.96	& 64.42$\pm$0.82	& 94.26$\pm$1.64	& 93.12$\pm$1.64	& 42.79$\pm$1.04	& 89.61$\pm$0.71	& 81.08$\pm$0.68	& 91.73$\pm$0.18 \\
\midrule
Improv.   & 3.88\%        & 6.70\%        & 1.15\%        & 0.82\%        & 1.04\%        & 0.90\%        & 0.15\%        & 2.80\%       \\
\bottomrule
\end{tabular}
}
\end{table*}

\subsection{New Benchmarks for Graph Heterophily}\label{apdix:new_bench}
For a more extensive evaluation across various domains, we also test SAF on five recently introduced datasets, including Mine., Tolo., Amaz., Roma., and Penn94~\cite{platonov2023critical,lim2021large}. In this context, we draw comparisons solely with MLP, GCN, APPNP, along with seven GNN models that have previously shown promising results in prior tasks, namely GPR-GNN, BernNet, ChebNetII, JacobiConv, NodeFormer, GloGNN++ and FE-GNN.
Table~\ref{tab:newheter_nc} lists the average classification accuracies, obtained over random splits provided by~\cite{platonov2023critical,lim2021large}, with a distribution of 50\%/25\%/25\% for training/validation/testing. 
In summary, SAF achieves significant performance gains of 
12.79\% and 4.03\%
on Minesweeper and Tolokers datasets, respectively, while maintaining competitiveness on the others.

\subsection{SAF with ChebNetII as Base Model}\label{apdix:chebNetAsBase}

To expand the versatility of our SAF framework, we introduced ChebNetII~\cite{he2022chebnetii} as an alternative base model, chosen for its adherence to the non-negative constraint, critical in our model design as stated in Section~\ref{sec:saf}. The rationale behind this choice is ChebNetII's use of Chebyshev interpolation for learning Chebyshev polynomials, where 
the constraint can be ensured by
keeping its learnable parameters 
$\{\gamma_j\}_{j=0}^K$ non-negative.
Our experiments, as shown in Table~\ref{tab:chebNetAsBase}, confirm that SAF can significantly enhances ChebNetII's performance, underscoring the framework's flexibility with different spectral filters. 
Interestingly, we observed that SAF, utilizing Bernstein polynomials (SAF-Bern), slightly surpasses its performance with Chebyshev polynomials (SAF-Cheb) in most datasets. The margin of improvement over the base model is also more pronounced with SAF-Bern. This phenomenon could be attributed to the $g_\phi(\lambda) \leq 1$ constraint within SAF (refer to Section~\ref{sec:saf}), necessitating the rescaling of filter functions by their maximum 
values. For Bernstein polynomials, this maximum is readily obtained as the largest polynomial coefficient $\max \{\phi_k\}_{k=0}^K$, as per Proposition 3. However, for Chebyshev polynomials, the best theoretical upper bound is the sum of absolute coefficients, $\sum_{k=0}^K |\phi_k|$, which is comparatively less precise. This difference may impact the quality of graph construction and, subsequently, compromise the model's performance. Exploring these nuances will be a focal point of our future research.

\begin{figure}[t]
\centering
\includegraphics[width=0.48\textwidth]{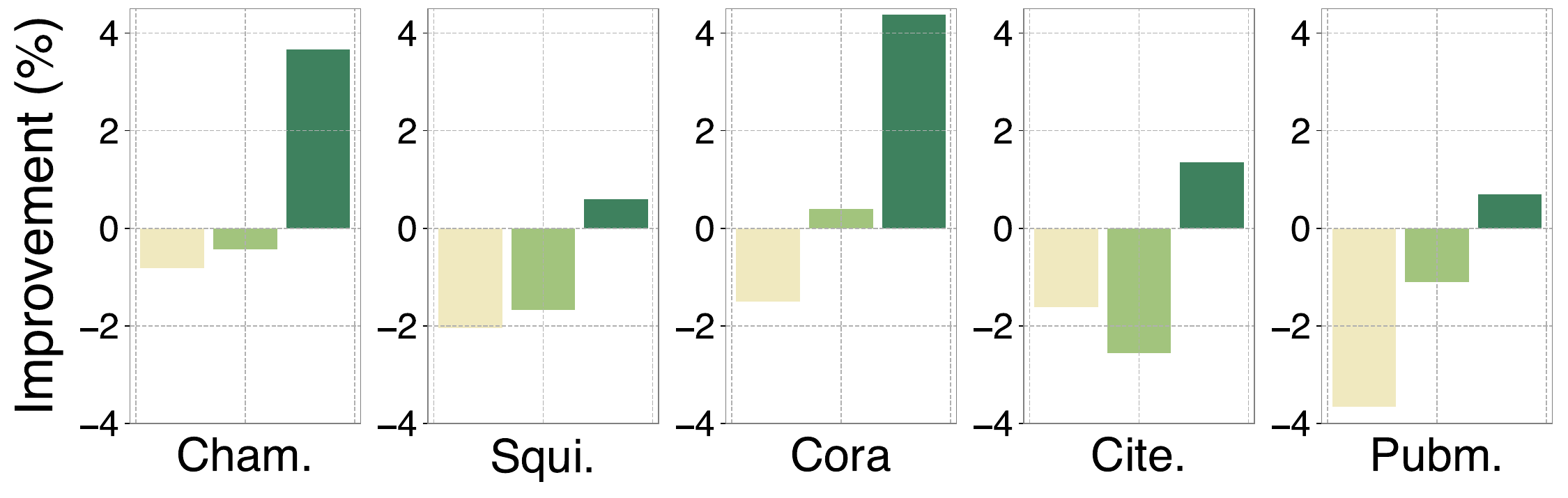}
\caption{Improved classification accuracy by GloGNN++, LRGNN and our SAF (from left to right) over the vanilla GCN under semi-supervised setting (2.5\%/2.5\%/95\%).}
\label{fig:sparse_advtg}
\end{figure}

\subsection{Supervision Dependence in GloGNN++ and LRGNN}\label{apdix:sparse_advantange}
This subsection empirically verifies the heavy reliance on label supervision for both GloGNN++~\cite{li2022finding} and LRGNN~\cite{liang2024predicting} models. In Fig.~\ref{fig:sparse_advtg}, we compare the improved classification accuracy of these models and our SAF upon the vanilla GCN under a semi-supervised setting. The figure illustrates that, despite competitive results under dense supervision (Table~\ref{tab:full_nc}), the performance of both GloGNN++ and LRGNN deteriorates, demonstrating negative optimization on many datasets, when the label rate drops approximately from 60\% to 2.5\%. This degradation highlights their dependency on high label rates for effective learning. In contrast, our SAF constantly delivers good performance regardless of the supervision level, showing its effectiveness in both high and low label rate scenarios.

\begin{table*}[t]
\caption{Time overheads (s).}\label{tab:timeOveheads}
\centering
\setlength\tabcolsep{10pt}
\resizebox{0.98\textwidth}{!}{
\begin{tabular}{l|ccccccc|cc}
\toprule 
\textbf{Method}
    & \textbf{Cham.} & \textbf{Squi.} & \textbf{Texas} & \textbf{Corn.} & \textbf{Actor} & \textbf{Cora}  & \textbf{Cite.} & \textbf{Pubm.} & \textbf{Penn94}  \\
\# Nodes           &2,277     &5,201     &183     &183     &7,600     &2,708     &3,327     &19,717     &41,554      \\
\# Edges           &36,101      &217,073      &309      &295      &33,544      &5,429      &4,732      &44,338      &1,362,229 \\
\midrule
BernNet           & 8.36  & 13.74 & 3.92  & 4.16  & 4.88  & 5.24  & 5.52  & 6.06  &24.05    \\
ChebNetII     & 22.82 & 30.73 & 11.47 & 9.64  & 14.88 & 19.96 & 16.14 & 36.91 &41.67  \\
FE-GNN     &3.99	  &45.89	  &1.03	  &0.84	  &0.98	  &2.54	  &2.06	  &6.78     &7.33 \\
NodeFormer     & 58.96 & 79.66 & 14.29 & 18.89 & 66.20 & 19.25 & 32.00 & 68.57 & 122.91 \\
GloGNN++ 	  & 35.63	& 68.31	 & 4.47	 & 3.00	 &73.13	 &32.68	&12.35	  &5266.53    	&3614.37\\
\midrule
SAF           & 11.55 & 18.78 & 4.38  & 4.70  & 5.36  & 6.04  & 6.12  & 18.43 &23.49 \\
Decomposition   &0.58	&1.59		&0.02		&0.02		&3.93	&1.00		&0.77	&21.34		&4.76 
\\
\bottomrule
\end{tabular}}
\end{table*}

\begin{table*}[t]
\caption{Space overheads (MB).}\label{tab:spaceOveheads}
\centering
\setlength\tabcolsep{10pt}
\resizebox{0.98\textwidth}{!}{
\begin{tabular}{l|ccccccc|cc}
\toprule 
\textbf{Method}
    & \textbf{Cham.} & \textbf{Squi.} & \textbf{Texas} & \textbf{Corn.} & \textbf{Actor} & \textbf{Cora}  & \textbf{Cite.} & \textbf{Pubm.} & \textbf{Penn94}  \\
\midrule
BernNet           & 72  & 232 & 5  & 5  & 292  & 64  & 152  & 1546  &1902    \\
ChebNetII     & 72  &231 &5 &5  & 291 & 63 & 152 & 1584 &1850  \\
FE-GNN     &1337	  &6919	  &23	  &10	  &416	  &302	  &740	  &2213     &5854 \\
NodeFormer     & 1522 & 3965 & 15 & 37 & 775 & 480 & 764 & 2119 & 3056 \\
GloGNN++ 	  & 290	&1525	 & 5	 & 5	 &2471	 &331	&607	  &17892    	&25260\\
\midrule
SAF           & 112 & 440 & 5  & 5  & 733  & 120  & 237  & 4515 &8491 \\
Decomposition   &141	&540		&1		&1		&1206	&140		&239	&7641		&4 
\\
\bottomrule
\end{tabular}}
\end{table*}

\subsection{Time and Space Overheads}\label{apdix:run_time}

\subsubsection{Eigendecomposition} 
Our SAF framework pre-computes eigendecomposition once per graph and reuses it in Eq.~(\ref{eq:theo_interplay}) in both training and inference. This aspect is crucial, as the forward-pass cost in model training often exceeds the preprocessing expense of eigendecomposition. To empirically validate this, we compare the time overheads of eigendecomposition with the training times of various models in Table~\ref{tab:timeOveheads}. Overall, we have following observations:
\textbf{1)} For datasets with a small number of nodes, the time consumed by decomposition is significantly less than the time required for model training;
\textbf{2)} For medium-sized graphs such as Pubmed, while the full decomposition time exceeds that of BernNet, it still maintains efficiency against more advanced GNNs such as ChebNetII, 
NodeFormer and GloGNN++; \textbf{3)} Moving to the large-scale graph, Penn94 (with 41,554 nodes and 1,362,229 edges), where only partial eigendecomposition with 100 extremal eigenvalues is considered, the computation time is markedly reduced compared to all the models. For space overheads in Table~\ref{tab:spaceOveheads}, similar patterns can be observed.

\subsubsection{Model Comparison} 
In Table~\ref{tab:timeOveheads}, we compare the running times of our SAF against two notable spectral GNNs (BernNet, ChebNetII), two non-local GNNs (NodeFormer, GloGNN++), 
and one unified GNN model (FE-GNN).
Generally, one can observe that SAF, while slightly slower than its base model, BernNet, due to the integration of non-local spatial aggregation, remains more efficient than or comparable to other SOTA methods, particularly those also employing non-local approaches such as NodeFormer and GloGNN++.

\section{Conclusion}
This study presents a cross-domain analysis on GNN models, offering a fresh perspective by rethinking spectral GNNs from a spatial lens. 
We reveal that spectral GNNs fundamentally leads the original graph to an adapted new one, which exhibits non-locality and accommodates signed edge weights to reflect label consistency among nodes. This insight leads to our proposed Spatially Adaptive Filtering (SAF) framework, enhancing spectral GNNs for more effective and versatile graph representation learning. While SAF adeptly captures long-range dependencies and addresses graph heterophily, we acknowledge two limitations of this work:~1) the non-negative constraints of the proposed SAF on graph filters might limit model expressiveness, indicating room for theoretical refinement;~2) although this study focuses on a node-level investigation, it raises intriguing questions about the implications of spectral GNNs at the graph-level in the spatial domain. Future work could expand this examination by relaxing theoretical constraints or exploring the cross-domain interplay from a broader graph-level viewpoint.

\section*{Acknowledgments}
The work was supported by the following: National Natural Science Foundation of China under No. 92370119, and 62376113.

\bibliography{saf}
\bibliographystyle{IEEEtran}

\end{document}